\documentclass{article}

\usepackage{microtype}
\usepackage{graphicx}
\usepackage{subfigure}
\usepackage{booktabs} 

\usepackage{hyperref}

\usepackage[accepted]{icml2024}

\usepackage{amsmath}
\usepackage{amssymb}
\usepackage{mathtools}
\usepackage{amsthm}

\usepackage[capitalize,noabbrev]{cleveref}

\theoremstyle{plain}
\newtheorem{theorem}{Theorem}[section]

\newtheorem{lemma}[theorem]{Lemma}
\newtheorem{corollary}[theorem]{Corollary}
\theoremstyle{definition}
\newtheorem{definition}[theorem]{Definition}
\newtheorem{assumption}[theorem]{Assumption}
\theoremstyle{remark}

\usepackage[textsize=tiny]{todonotes}

\icmltitlerunning{On the Second-Order Convergence of Biased Policy Gradient Algorithms}

\begin{document}

\twocolumn[
\icmltitle{On the Second-Order Convergence of Biased Policy Gradient Algorithms}



\icmlsetsymbol{equal}{*}

\begin{icmlauthorlist}
\icmlauthor{Siqiao Mu}{yyy}
\icmlauthor{Diego Klabjan}{sch}
\end{icmlauthorlist}

\icmlaffiliation{yyy}{Department of Engineering Sciences and Applied Mathematics, Northwestern University, Evanston, IL}
\icmlaffiliation{sch}{Department of Industrial Engineering and Management Sciences, Northwestern University, Evanston, IL}

\icmlcorrespondingauthor{Siqiao Mu}{siqiaomu2026@u.northwestern.edu}
\icmlcorrespondingauthor{Diego Klabjan}{d-klabjan@northwestern.edu}

\icmlkeywords{Machine Learning, ICML, saddle points, nonconvex optimization, optimization, reinforcement learning, policy gradient}

\vskip 0.3in
]



\printAffiliationsAndNotice{\icmlEqualContribution} 

\begin{abstract}
Since the objective functions of reinforcement learning problems are typically highly nonconvex, it is desirable that policy gradient, the most popular algorithm, escapes saddle points and arrives at second-order stationary points. Existing results only consider vanilla policy gradient algorithms with unbiased gradient estimators, but practical implementations under the infinite-horizon discounted reward setting are biased due to finite-horizon sampling. Moreover, actor-critic methods, whose second-order convergence has not yet been established, are also biased due to the critic approximation of the value function. We provide a novel second-order analysis of biased policy gradient methods, including the vanilla gradient estimator computed from Monte-Carlo sampling of trajectories as well as the double-loop actor-critic algorithm, where in the inner loop the critic improves the approximation of the value function via TD(0) learning. Separately, we also establish the convergence of TD(0) on Markov chains irrespective of initial state distribution.
\end{abstract}

\section{Introduction}

In the standard reinforcement learning framework, an agent interacts with an environment according to some policy, which dictates the best actions to take given the state of the environment. The ultimate goal of the agent is to adopt a policy that maximizes some measure of cumulative reward. To efficiently search for the optimal policy, policy gradient methods optimize a policy parameter $\theta$ through updates that approximate the gradient of the objective function with respect to $\theta$. These algorithms can be fast and flexible, and under mild assumptions, they share convergence properties with mainstream gradient descent algorithms.

The policy gradient can be estimated by way of the policy gradient theorem, which enables the straightforward application of existing techniques in gradient-based optimization towards understanding PG convergence. The theorem provides a formula for the exact gradient of the objective function as the expectation of the state-action value function over the discounted state-action measure \cite{1992SuttonOG}. In practice, the gradient is estimated through two common approaches: the ``vanilla" approach, where the gradient is approximated via Monte-Carlo sampling, or the ``actor-critic" approach, where the policy parameter $\theta$, called the ``actor parameter," is updated simultaneously alongside a ``critic parameter" $w$, which parametrizes the state-action value function. The critic parameter is typically updated via bootstrapping temporal difference methods such as TD(0), although the actor updates may also be bootstrapped \cite{konda, BHATNAGAR20092471NAC}. 

Although it is well-established that policy gradient converges to first-order stationary points where the norm of the gradient is approximately zero \cite{1992SuttonOG, yuan21vanilla,agarwal}, it is of interest whether policy gradient algorithms yield second-order stationary points (local maxima) as opposed to saddle points. This is because the function landscape of RL problems can be highly nonconvex and features suboptimal stationary points even in very simple examples \cite{agarwal, bhandariglobalPGforreal, zhang2020global}. We can utilize seminal results in nonconvex optimization that show that stochastic gradient descent can leverage randomness to escape saddle points either with added noise \cite{gePSGD2015, jinsaddleSGD} or as long as there is a component of noise in the direction of curvature \cite{daneshmand, vlaski2019secondorder}.

However, unlike stochastic gradient descent, practical implementations of policy gradient are frequently biased. For the discounted infinite-horizon reward objective function, the vanilla policy gradient estimator is biased due to truncation of sampled trajectory horizons \cite{yuan21vanilla}. In addition, actor-critic algorithms introduce a second source of bias in the critic's approximation of the value function. This therefore requires novel techniques for controlling and bounding the bias. In comparison, existing works on second-order convergence of policy gradient only consider unbiased gradient estimators; for example, \cite{zhang2020global} require artificially constructing an unbiased gradient estimator via Q-sampling as well as periodic step size enlargement, and \cite{yang2020pgsosp} assume the gradient estimator is unbiased. Both works also only consider the vanilla policy gradient algorithm and not the actor-critic scheme.

In this paper, we address the aforementioned gap by showing that biased policy gradient methods, including both vanilla policy gradient and actor-critic methods, can escape saddle points. Borrowing from the analyses of \cite{vlaski2019secondorder} regarding unbiased stochastic gradient descent, we tackle the biased policy gradient setting by showing that the components of the bias can be sufficiently controlled to yield second-order convergence guarantees.

\textbf{Related Work}

\textit{Escaping Saddle Points.} As referenced earlier, escaping saddle points has become a central research topic in nonconvex optimization in the last few years, with natural extensions to the policy gradient setting. In addition to the seminal works showing the second-order convergence of gradient descent and stochastic gradient descent \cite{gePSGD2015, jinsaddleSGD, Jin2017howtoescape, daneshmand, vlaski2019secondorder}, an additional body of work has focused on second-order methods that use Hessian information to arrive at second-order stationary points faster. Some examples of reinforcement learning approaches in this direction include \cite{2019shenhessianaided, WANG2022109687, khorasani2023efficiently}. However, these works only consider unbiased gradient estimators. For an in-depth comparison of sample complexities, see Appendix \ref{app:secondorder}.

\textit{Global Convergence.} Separate from our line of work, there are several ``global convergence" results for policy gradient and actor-critic algorithms that ensure convergence to a global optimum for specific policy parametrization or function structure. For instance, global convergence is established for the following settings: tabular or tabular softmax policy parametrization with exact gradients \cite{agarwal, bhandari21aglobalconvPG}, objective functions that satisfy the gradient domination property \cite{bhandariglobalPGforreal}, linear quadratic and nearly linear-quadratic control systems \cite{yanglqr, han2022policy}, and overparametrized neural networks \cite{wangzhaoranNeuralPG, FuAC}. The premise of ``global convergence" also requires some assumption that the proposed policy parametrization can approximate the optimal policy with arbitrary precision, i.e., the optimal policy lies within the policy class. In comparison, our second-order convergence guarantees pertain to finding the best policy parameter $\theta$ under a generic policy parametrization, for policy gradient algorithms with noisy and biased updates. For an in-depth discussion of global convergence rates, see Appendix \ref{app:global}.

\textit{Vanilla Policy Gradient.} In this work, we refer to policy gradient algorithms that estimate the gradient via Monte-Carlo sampling of trajectories as ``vanilla policy gradient." Early formulas  include REINFORCE \cite{Williams2004SimpleSG} as well as GPOMDP \cite{bartlettbaxter}, a version of REINFORCE that enjoys reduced variance \cite{PETERS2008682} by employing the "reward-to-go" trick. Our work pertains to the GPOMDP estimator. Both REINFORCE and GPOMDP are unbiased estimators of objectives with deterministic, fixed horizons, but they are biased estimators of the infinite-horizon discounted reward due to truncation \cite{yuan21vanilla}. 

\textit{Actor-critic Algorithms.} The asymptotic convergence of actor-critic algorithms was first established in \cite{konda} for gradient actor updates and bootstrapped critic updates, and in \cite{BHATNAGAR20092471NAC} for bootstrapped actor and critic updates. Since then, finite-time convergence has been established for a variety of actor-critic frameworks, although to the best of our knowledge no second-order convergence result exists. In this work we consider double-loop actor-critic algorithms where the critic parameter undergoes TD(0) updates in the inner loop and the actor parameter undergoes gradient updates in the outer loop. Several works have shown first-order convergence of these algorithms with various additional settings; \cite{yangACiid} consider i.i.d sampling in the actor and the critic, \cite{qiuACzhaoran} consider i.i.d. policy samples and critic sampling from a stationary Markov chain, and \cite{xuAC2020nestedloop} study mini-batch Markovian sampling for controlling the bias error. Separately, \cite{wangzhaoranNeuralPG} establish global convergence for double-loop algorithms where the actor and critic functions are compatible overparametrized neural networks. In comparison, we consider linear critic parametrization and arbitrary actor parametrization with Markovian sampling, and we do not require the Markov chain to be stationary. 

Recently, two-timescale \cite{xuAC, wuACquanquan} and single-timescale \cite{FuAC, olshevsky2023smallgain} actor-critic algorithms have shown a slight performance advantage over double-loop algorithms, although existing works still only analyze their first-order or global convergence under specific function parametrizations, while we focus on second-order convergence. 

\textit{Temporal Difference Algorithms.} Actor-critic algorithms typically feature some bootstrapping element; in particular, we consider actor-critic algorithms where the critic updates via TD(0) learning. The finite-time convergence of TD(0) has been established recently under independent and identically distributed sampling \cite{dalalTD, kumar} as well as under Markovian sampling \cite{liusplitting, bhandariTD}. The latter is more relevant to the RL setting; however, both \cite{bhandariTD} and \cite{liusplitting} assume the Markov chain begins in the stationary distribution, which is unrealistic for practical implementations such as in actor-critic algorithms. 

\textbf{Our Contributions} 

Our contributions are as follows.
\begin{itemize}
    
    \item We provide the first finite-time convergence analysis of vanilla policy gradient with biased gradient estimator to $\epsilon$-second-order stationary points. This results in a sample complexity of $\tilde{O}(\epsilon^{-6.5})$ iterations, where $\tilde{O}(\cdot)$ hides logarithmic dependencies. We note that this is a stronger result than $\tilde{O}(\epsilon^{-9})$ from \cite{zhang2020global} and a weaker result than $\tilde{O}(\epsilon^{-4.5})$ from \cite{yang2020pgsosp}, both of which only analyze unbiased gradient estimators.

    \item We provide the first finite-time convergence analysis of an actor-critic policy gradient algorithm to $\epsilon$-second-order stationary points. We show that our double-loop actor-critic algorithm, where the critic updates via $TD(0)$ and the actor updates via policy gradient, arrives at an $\epsilon$-second-order stationary point in $\tilde{O}(\epsilon^{-6.5})$ outer loop iterations with $\tilde{O}(\epsilon^{-8})$ inner-loop TD(0) steps.  In contrast to existing first-order analyses of actor-critic algorithms, we allow for Markovian sampling in both the actor and the critic. 

    \item Of separate interest, we provide the first finite-time convergence analysis of the classic TD(0) algorithm on nonstationary Markov chains; i.e., we do not assume that the initial state distribution is the stationary distribution of the Markov chain. This allows realistic analyses of the actor-critic setting, where we have no guarantee that after every policy update the new underlying Markov chain is in its stationary distribution. We show that for $K$ constant timesteps $\alpha = \frac{1}{\sqrt{K}}$ and exponential mixing, the algorithm converges at the rate of $O(\frac{1}{\sqrt{K}}) + O(\frac{1}{K})$, as opposed to $O(\frac{1}{\sqrt{K}})$ when starting from the stationary distribution.
    
\end{itemize}

The structure of the paper is as follows. In Section 2, we formalize the problem and introduce notation used in the rest of the paper. In Section 3, we establish second-order convergence for biased policy gradient estimators and apply our results to vanilla policy gradient. In Section 4, we present our new analysis of the TD(0) algorithm, which is incorporated to bound the critic approximation bias and show second-order convergence of the actor-critic algorithm. Finally, in Section 5, we summarize our results and discuss the next steps of our work.

\section{Problem Formulation}

\subsection{Markov Decision Process}

We define the Markov decision process as a quadruple $( \mathcal{S}, \mathcal{A}, \mathcal{P}, \mathcal{R})$, where $\mathcal{S}$ is the state space, $\mathcal{A}$ is the action space, $\mathcal{P}(s' | s, a)$ is the transition probability from state $s$ to state $s'$ by taking action $a$, and $\mathcal{R}(s, a)$ is the reward function for performing action $a$ in state $s$. The agent is trying to learn a stochastic policy $\pi : \mathcal{S} \to \Delta(\mathcal{A})$, where $\Delta (\mathcal{A})$ is the space of probability distributions over $\mathcal{A}$, such that $\pi(a | s)$ is the probability that the agent performs action $a$ given state $s$. As the agent interacts with the environment, it generates a sequence of states, actions, and rewards referred to as a trajectory 
$\tau = \{ s_0, a_0, s_1, a_1, ... \}.$ The trajectory is sampled from the probability distribution $p( \cdot | \pi)$, which describes the probability of a trajectory generated by some policy $\pi$, where $a_k \sim \pi(\cdot | s_k)$ and $s_{k+1} \sim \mathcal{P}(\cdot | s_k, a_k)$. 

We want to learn a policy $\pi$ that maximizes the expected infinite-horizon discounted reward, $J$. For policy gradient algorithms, we parameterize the policy $\pi$ with some parameter $\theta \in \mathbb{R}^M$ so that $J$ is a function of $\theta$ to obtain
\begin{equation*}
\begin{aligned}
    J(\theta) &= \mathbb{E}_{s_0 \sim \rho_0(\cdot), \tau \sim p( \cdot | \pi_\theta )} [\sum_{k = 0}^\infty \gamma^k \mathcal{R}(s_k, a_k) ] 
\end{aligned}
\end{equation*}
where $\gamma \in (0, 1)$ is the discount factor and the expectation is taken with respect to an initial state distribution $s_0 \sim \rho_0(\cdot)$ and a stochastic policy $\pi_\theta$ under which trajectories are sampled $\tau \sim p(\cdot | \pi_\theta)$. The RL problem is to find an optimal policy parameter $\theta^*$ such that $\theta^* = \arg \max_\theta J(\theta)$.

We also define the state value function and state-action value function  as $V^\pi (s)~=~\mathbb{E}_\pi[\sum_{k=0}^\infty \gamma^k \mathcal{R}(s_k, a_k) | s_0~=~s]$  and $Q^\pi(s, a)~=~\mathbb{E}_\pi[\sum_{k=0}^\infty \gamma^k \mathcal{R}(s_k, a_k) | s_0~=~s, a_0~=~a]$ respectively,
such that the objective can be alternatively formulated as
$J(\theta) = \mathbb{E}_{s \sim \rho_0}[V^{\pi_\theta}(s)].$
Finally, we reference the discounted state-weighting measure as 
$d^{\pi_\theta, \rho_0} (s)~=~\sum_{k = 0}^\infty \gamma^k \mathbb{E}_{\rho_0}[Pr(s_k~=~s~| s_0, \pi_\theta)].$

\subsection{Policy Gradient Algorithm}
\begin{algorithm}
\caption{Biased Policy Gradient Algorithm}\label{alg:pseudo}
\begin{algorithmic} 
\STATE {\bfseries Input:} initial policy parameters $\theta_0$
\FOR{$t = 0, 1, 2, ... T-1$}
    \STATE Sample a trajectory $\tau_t$ of length $H$ under $\pi_{\theta_t}$, 
    \begin{equation*}
        \tau_t = \{s_0, a_0, s_1, a_1,... s_{H-1}, a_{H-1} \}
    \end{equation*}
    \STATE where $s_0 \sim \rho_0(s)$ 
    \STATE Compute $\hat{G}(\theta_t; \tau_t)$ via \textbf{Algorithm 2} or \textbf{Algorithm 3}
    
    \STATE Update $\theta_{t+1} = \theta_t + \mu \hat{G}(\theta_t; \tau_t)$
\ENDFOR
\end{algorithmic}
\end{algorithm}
In this work we consider policy gradient algorithms of the form shown in Algorithm \ref{alg:pseudo}. We note that these gradient algorithms can be mini-batched for additional variance reduction, but our analysis does not rely on batching and we omit this notation for simplicity. At each time step $t+1$, a random trajectory $\tau_t$ is sampled and $\hat{G}(\theta_t ; \tau_t)$, a biased estimator of the true policy gradient $\nabla J(\theta_t)$, is computed. Then the policy parameter $\theta$ is updated with step-size $\mu$ as follows
\begin{equation}
\label{eq:generalpg}
    \theta_{t+1} = \theta_t + \mu \hat{G}(\theta_t ; \tau_t).
\end{equation}
Let $G(\theta_t) = \mathbb{E}_{\tau}[\hat{G}(\theta_t ; \tau)]$, where the expectation is taken with respect to the sampled trajectory $\tau$. Then in order to analyze the convergence of Algorithm 1, we decompose the updates in terms of the noise $\xi_{t+1}$ and bias $d_{t+1}$ to obtain
\begin{equation*} \begin{aligned}
     \theta_{t + 1} &= \theta_{t} + \mu \nabla J (\theta_{t}) + \mu \xi_{t + 1} + \mu d_{t + 1},
\end{aligned} \end{equation*}
where $\xi_{t + 1} = \hat{G}(\theta_{t} ; \tau_t) - G(\theta_{t})$ represents a zero-mean noise term induced from sampling the trajectory of length $H$ and $d_{t + 1} =  G(\theta_{t}) - \nabla J (\theta_{t})$ represents the bias induced by the gradient estimator algorithm. Specifically, we denote by $\{\mathcal{F}_t \}_{t \geq 0}$ the filtration generated by the random process $\{ \theta_t \}_{t \geq 0}$ such that $\theta_t$ is measurable with respect to $\mathcal{F}_t$. Then the gradient noise process $\{\xi_t\}_{t \geq 0}$ satisfies 
\begin{equation*}
\mathbb{E}[\xi_{t+1} | \mathcal{F}_{t} ] = \mathbb{E}[\hat{G}(\theta_{t} ; \tau_t) - G(\theta_{t}) | \mathcal{F}_{t} ] = 0.
\end{equation*}
 Although $d_{t+1}$ might be random or deterministic depending on the gradient estimator algorithm, it is generated by a different process than $\xi_{t+1}$.
In the sequel, our analyses rely on assumptions on $\xi_{t+1}$ and $d_{t+1}$.

\subsection{Second-Order Stationary Points}
We define second-order stationary points and approximate second-order stationary points as follows \cite{Jin2017howtoescape, nesterovcubic}. 

\begin{definition}
    For the twice-differentiable function $J(\theta)$, $\theta$ is a second-order stationary point if $\lVert \nabla J(\theta) \lVert = 0$ and $\lambda_{max} (\nabla^2 J(\theta)) \leq 0$.
In addition, if $\nabla^2 J(\theta)$ is $\chi$-Lipschitz, $\theta$ is an $\epsilon $-second order stationary point of $J(\theta)$ if
$\lVert \nabla J(\theta) \lVert ~\leq~\epsilon$ and 
$\lambda_{max} (\nabla^2 J(\theta)) \leq \sqrt{\chi \epsilon}.$
\end{definition}

In line with the \textit{strict-saddle} definition introduced in \cite{gePSGD2015} and later used in \cite{jinsaddleSGD, Jin2017howtoescape, daneshmand}, we focus on escaping saddle points with at least one strictly positive eigenvalue. We divide the parameter space of the objective function $J(\theta)$ into regions where the gradient is large or small with respect to the step-size $\mu$, which we assume to be a small hyperparameter. We define the following sets
$$\mathcal{G} = \left\{ \theta: \lVert \nabla J(\theta) \rVert ^2 \geq  \mu \ell \left( 1 + \frac{1}{\delta} \right) \right\},$$
$$\mathcal{H} = \{ \theta : \theta \in \mathcal{G}^C, \lambda_{max} ( \nabla^2 J(\theta)) \geq \omega \} ,$$
$$\mathcal{M} = \{ \theta : \theta \in \mathcal{G}^C, \lambda_{max}(\nabla^2 J(\theta)) < \omega \}, $$
where $\ell > 0$ is a parameter depending on the problem and $\delta > 0$, $\omega > 0$ are parameters depending on the desired accuracy of the algorithm that we choose later on. The set $\mathcal{G}^C$ represents approximate first-order stationary points, and within $\mathcal{G}^C$, the region $\mathcal{H}$ represents ``strict-saddle" points, whereas the region $\mathcal{M}$ represents approximately second-order stationary points. Specifically, points in $\mathcal{M}$ represent the set of local maxima with respect to first and second-order information.

\section{Second-Order Convergence of Vanilla Policy Gradient}

In this section, we establish the second-order convergence of biased policy gradient in general and apply it to vanilla policy gradient. 
We begin with the original policy gradient theorem \cite{1992SuttonOG}, which states
$$\nabla J(\theta) = \sum_{s \in \mathcal{S}} d^{\pi_\theta, \rho_0}(s) \sum_{a \in \mathcal{A}} \pi_\theta(a | s) \nabla \log \pi_\theta (a | s) Q^{\pi_\theta}(s, a),$$
which is equivalent to the temporal summation
$$\nabla J(\theta) = \mathbb{E}_{\pi_\theta, \rho_0} [\sum_{k = 0}^\infty \sum_{t = k}^\infty \gamma^{t} \nabla \log \pi_\theta(a_k|s_k) \mathcal{R}(s_{t}, a_{t}) ].$$
For further details on the connection between these two formulations, see Appendix \ref{PGT}.

Through Monte-Carlo sampling of trajectories $\tau$ of fixed length $H$, we construct the GPOMDP gradient estimator \cite{bartlettbaxter} as follows
$$\hat{G}^{VPG}(\theta ; \tau) = \sum_{h = 0}^{H-1} \nabla \log \pi_\theta (a_h | s_h) \sum_{i = h}^{H-1} \gamma^i \mathcal{R} ( s_i, a_i) . $$
\begin{algorithm} 
\caption{Vanilla Policy Gradient Estimator}\label{alg:vanilla}
\begin{algorithmic} 
\STATE \textbf{Input: }policy parameter $\theta_t$, trajectory $\tau_t=\{s_0, a_0, s_1, a_1,... s_{H-1}, a_{H-1} \}$
\STATE Compute gradient estimator from $\tau_t$:
$$\hat{G}^{VPG}(\theta_t ; \tau_t) = \sum_{h = 0}^{H-1} \nabla \log \pi_\theta (a_h | s_h) \sum_{i = h}^{H-1} \gamma^i \mathcal{R} ( s_i, a_i)  $$
\STATE \textbf{return} $\hat{G}^{VPG}(\theta_t ; \tau_t)$
\end{algorithmic}
\end{algorithm}

This yields the ``vanilla" policy gradient algorithm as outlined in Algorithm \ref{alg:vanilla}. We define the truncated or finite-horizon objective $J_H(\theta)~=~\mathbb{E}_{\pi_\theta, \rho_0}  [\sum_{t = 0}^{H - 1} \gamma ^t \mathcal{R}(s_t, a_t)]$. As observed in \cite{yuan21vanilla, zhang2020global, wu22sensitivity}, since we cannot practically sample infinite-horizon trajectories, $\hat{G}^{VPG} (\theta ; \tau)$ is a \textit{biased} gradient estimator of $J(\theta)$ and an unbiased gradient estimator of $J_H(\theta)$, such that
$$\mathbb{E}[\hat{G}^{VPG}(\theta ; \tau)] = \nabla J_H(\theta) \neq \nabla J(\theta).$$


\subsection{Assumptions} 
We require the following assumptions for the convergence of biased policy gradient algorithms.  

\begin{assumption}
\label{ass_1boundedreward}
    The following conditions hold for all $(s, a) \in \mathcal{S} \times \mathcal{A}$ and $\theta$:
    \begin{enumerate}
        \item The rewards are bounded such that there exists $\mathcal{R}_{max} > 0$ with $|\mathcal{R}(s, a) | \leq \mathcal{R}_{max}$.
        \item The policy score function $\nabla \log \pi_\theta (a | s)$ exists and its norm is bounded by a constant $G > 0$ such that
$\lVert \nabla \log \pi_\theta (a | s) \rVert \leq G$.
        \item The Jacobian of the score function exists and its norm is bounded by a constant $B > 0$ such that $\lVert  \nabla^2 \log \pi_{\theta} (a | s) \rVert < B$, and it is Lipschitz continuous such that for all $\theta_1, \theta_2$, we have
        $$\lVert \nabla^2 \log \pi_{\theta_1} (a | s) - \nabla^2 \log \pi_{\theta_2} (a | s) \rVert \leq \iota \lVert \theta_1 - \theta_2 \rVert.$$
    \end{enumerate}
\end{assumption}

In Assumption~\ref{ass_1boundedreward}, we assume that the reward and the policy log gradient are bounded, assumptions first put forward in \cite{PAPINI} and since widely adopted in many theoretical analyses of policy gradient \cite{zhuheteroFLrL, zhang2020global, yang2020pgsosp}. Assumption \ref{ass_1boundedreward} is satisfied by commonly used policy parametrizations, including the softmax policy parametrization $\pi_{\theta}(a | s) = \frac{e^{h(s, a, \theta)}}{\sum_b e^{h(s, b, \theta)}}$. The action preferences $h(s, a, \theta)$ can be parametrized via deep neural networks or other functions of the feature vectors $\phi(s, a)$. The assumption is also satisfied by Gaussian policies such as $\pi_\theta(a | s) \sim \mathcal{N}(\phi(s)^\intercal \theta, \sigma^2)$ if the parameter $\theta$ lies in some bounded set and the actions and the feature vectors $\phi(s)$ are bounded \cite{zhang2020global}. In addition, Assumption \ref{ass_1boundedreward} has several important implications as follows.

\begin{lemma}
\label{lemma_lipschitz}
(Lemma 3.2 of \cite{zhang2020global}) 
    The score function $\nabla \log \pi_{\theta}(a | s)$ is $B$-Lipschitz continuous. Moreover, the policy gradient $\nabla J(\theta)$ is Lipschitz continuous such that for all $\theta_1$, $\theta_2$, we have
    $$\lVert \nabla J(\theta_1) - \nabla J (\theta_2) \rVert \leq L \lVert \theta_1 - \theta_2 \rVert $$
    where $L = \frac{\mathcal{R}_{max} B}{(1 - \gamma)^2} + \frac{(1 + \gamma) \mathcal{R}_{max} G^2}{(1 - \gamma)^3}$.
\end{lemma}

\begin{lemma}
\label{lemma_hessian}
    (Lemma 5.4 from \cite{zhang2020global}) The Hessian matrix of $J(\theta)$ is Lipschitz continuous such that for all $\theta_1$, $\theta_2$, we have 
    $$\lVert \nabla^2 J(\theta_1) - \nabla^2 J(\theta_2) \rVert \leq \chi \lVert \theta_1 - \theta_2 \rVert $$
    where $\chi = \frac{\mathcal{R}_{max} G B}{(1 - \gamma)^2} + \frac{\mathcal{R}_{max} G^3 (1 + \gamma)}{(1 - \gamma)^3} + \frac{\mathcal{R}_{max} G}{1 - \gamma} \cdot \max \{ B, \frac{G^2 \gamma}{1 - \gamma}, \frac{\iota}{G}, \frac{B \gamma}{1 - \gamma}, \frac{G^2 (1 + \gamma) + B (1 - \gamma) \gamma}{(1 - \gamma)^2}\}$.
\end{lemma}



Finally, we require assumptions on the noise of the algorithm that allows the iterates to escape saddle points. 

\begin{assumption}
\label{ass_5lipcovariance}
The covariance matrix of the noise $\xi_{t+1}$ generated at iterate $\theta_t$, defined as
    $R_\xi(\theta_{t}) = \mathbb{E}[\xi_{t+1} \xi_{t+1} ^\intercal | \mathcal{F}_{t}]$, is Lipschitz such that
    $$\lVert R_\xi(\theta_1) - R_\xi(\theta_2)\lVert \leq \beta_R \lVert\theta_1 - \theta_2\lVert^\nu$$
    where $0 < \nu \leq 4$ and $\beta_R > 0$.
\end{assumption}

\begin{assumption}
\label{ass_6noisecurv}
Let $\nabla^2 J(\theta) = V_\theta \Lambda_\theta V_\theta^T$ be the eigendecomposition of the Hessian matrix at $\theta$ where the eigenvalues and eigenvectors are ordered as follows
\[
V_\theta = \begin{bmatrix}
V_\theta^{> 0} & V_\theta^{\leq 0} \\
\end{bmatrix},
\qquad
\Lambda_\theta = \begin{bmatrix}
\Lambda_\theta^{> 0} & 0 \\
0 & \Lambda_\theta^{\leq 0} 
\end{bmatrix}\]
where $\Lambda_\theta^{> 0} > 0$ and $\Lambda_\theta^{\leq 0} \leq 0$. Then, we assume that there exists $\sigma_{l}^2 > 0$ such that for any approximate strict-saddle point $\theta_t \in \mathcal{H}$ 
$$\lambda_{min} ((V_{\theta_t}^{>0})^\intercal \mathbb{E}[\xi_{t+1} \xi_{t+1}^\intercal ]V_{\theta_t}^{>0} ) \geq \sigma_{l}^2.$$
\end{assumption}

Assumption \ref{ass_5lipcovariance} states that the covariance matrix is Lipschitz, an assumption proposed in \cite{vlaski2019secondorder, vlaskiSOFL}. Assumption \ref{ass_5lipcovariance} requires that the covariance of noise does not change too much between iterates, which can be ensured by restricting our parameter search space to a compact set, which reflects common practices. Assumption \ref{ass_6noisecurv}, or the condition of ``correlated negative curvature," states that there must be a component of noise in the direction of negative curvature in order for the noise to help the iterates escape the saddle point. It is necessary for most second-order convergence analyses \cite{daneshmand, zhang2020global, yang2020pgsosp}.
Like \cite{vlaski2019secondorder}, we note that although Assumption \ref{ass_6noisecurv} is a technical requirement for convergence, the condition can be achieved by simply adding isotropic noise at each parameter update iteration like in \cite{gePSGD2015, jinsaddleSGD}.

\subsection{Main Result}

We first present the following theorem for general biased policy gradient, which shows that if the gradient estimator and bias are sufficiently bounded, we can conclude second-order convergence. Theorem \ref{theorem:main} adapts Theorem 3 from \cite{vlaski2019secondorder} regarding second-order convergence of unbiased stochastic gradient descent. Our approach follows the framework proposed in \cite{vlaskiSOFL} for showing second-order convergence of federated learning, a form of biased stochastic gradient descent. 

\begin{theorem}
\label{theorem:main}
 For the iterates $\theta_t$ of Algorithm 1, suppose that Assumptions \ref{ass_1boundedreward}-\ref{ass_6noisecurv} hold and for $\sigma > 0$, $D > 0$, the following hold
\begin{equation}
\label{eq:thm1first}
    \lVert \hat{G}(\theta_t ; \tau_t) \rVert \leq \sigma
\end{equation}
\begin{equation}
\label{eq:momentbound4}
    \mathbb{E}[\lVert d_{t+1}\rVert^4| \mathcal{F}_{t} ] \leq D^4 \mu^4
\end{equation} 
\begin{equation}
\label{eq:momentboundboth}
    \mathbb{E}[\lVert d_{t+1}\rVert^2 \lVert \xi_{t+1} \rVert^2 | \mathcal{F}_{t} ] \leq \sigma^2\mathbb{E}[\lVert d_{t+1}\rVert^2 | \mathcal{F}_t].
\end{equation} 
Let $\ell = L \sigma^2 - D^2 \mu$. Then with probability $1 - \delta$, Algorithm 1 yields $\theta_{T} \in \mathcal{M}$ such that $\lVert\nabla J(\theta_{T}) \lVert ^2 \leq \mu \ell (1 + \frac{1}{\delta})$ and $\lambda_{max} (\nabla^2 J(\theta_{T})) \leq  \omega $ in $T$ iterations, where
$$T \leq \frac{4\mathcal{R}_{max}}{\mu^2  (1 - \gamma) (L\sigma^2 + D^2 \mu) \delta } \cdot 
\mathcal{T}$$
$$\mathcal{T} = \frac{\log (2 M \frac{\sigma^2}{\sigma^2_l} + 1)}{\log(1 + 2 \mu \omega)}.$$
\end{theorem} 

\textit{Proof.} See Appendix \ref{theorem2}.

Therefore, for both the vanilla and actor-critic settings, our key challenge is establishing the conditions (\ref{eq:thm1first}) (\ref{eq:momentbound4}) (\ref{eq:momentboundboth}). By applying Theorem \ref{theorem:main} and choosing $\mu$ to be small enough with respect to $\epsilon$, we can arrive at the following theorem establishing convergence of vanilla policy gradient to $\epsilon$-second order stationary points, specifying the required horizon of each sampled trajectory.

\begin{theorem}
\label{theorem:1epsilon}
 Suppose Assumptions \ref{ass_1boundedreward}-\ref{ass_6noisecurv} hold and let $\epsilon > 0$. For $\mu < \frac{\epsilon^2 \delta}{L\sigma^2 + D^2}$ where $D = \frac{ G \mathcal{R}_{max}}{1 - \gamma}$ and $\sigma = \frac{G \mathcal{R}_{max}}{(1 - \gamma)^2}$, we have with probability $1 - \delta$ that Algorithm 1 with vanilla policy gradient estimator computed via Algorithm \ref{alg:vanilla} and  $H = O(\log(\epsilon^{-2}))$ reaches an $\epsilon$-second order stationary point in $\tilde{O}(\epsilon^{-6.5})$ iterations, where $\tilde{O}(\cdot)$ hides logarithmic dependencies. 
\end{theorem}
In Theorem \ref{theorem:1epsilon}, $O(\cdot)$ hides dependency on $\gamma$, and $\tilde{O}(\cdot)$ hides dependencies on $L$, $G$, $\mathcal{R}_{max}$, $\gamma$, $M$, $\sigma_l$, $\chi$, $\delta$.

The detailed proof of Theorem \ref{theorem:1epsilon} is provided in Appendix \ref{app:VPG}. The proof sketch is as follows.
In order to apply Theorem \ref{theorem:main}, we first show  that the gradient estimator is uniformly bounded based on our assumptions (Lemma \ref{lemma:boundednoise}). This also implies that the second and fourth moment of the noise are also bounded. We then establish that the gradient bias due to truncation is deterministically bounded and decays as the trajectory horizon $H$ increases (Lemma \ref{lemma:biasmu}). 
For large enough $H$, the bias error is proportional to $\mu$, allowing us to bound away its effects.

\section{Second-Order Convergence of Actor-Critic Policy Gradient}

Now we consider the second-order convergence of actor-critic policy gradient algorithms. Actor-critic methods can reduce the variance of policy gradient methods by separately learning to approximate the state-action value function $Q^\pi$. With some algebraic manipulation, the policy gradient can be expressed as follows
$$\nabla J (\theta) = \mathbb{E}_{{\pi_\theta}, \rho_0} [\sum_{k = 0}^\infty \gamma^k \nabla \log \pi_\theta(a_k|s_k) Q^{\pi_\theta} (s_k, a_k)].$$
This motivates the construction of the following biased gradient estimator from a trajectory $\tau$ of length $H$
$$\hat{G}^{AC}(\theta ; \tau) = \sum_{k = 0}^H \gamma^k \nabla \log \pi_\theta(a_k|s_k) Q_w (s_k, a_k),$$
where $Q_w(s, a)$ is a function with parameter $w$ that approximates  $Q^{\pi_\theta}(s, a)$. Note that $\hat{G}^{AC}(\theta ; \tau)$ depends on $w$, although we omit this dependency in notation. In contrast to vanilla policy gradient, the actor-critic gradient estimator has two sources of bias: bias from the horizon truncation that depends on $H$, and bias from the critic approximation $Q_w$. Our following analysis addresses both. Let 
\begin{equation*}
    \begin{aligned}
        G(\theta) &= G_H(\theta) = \mathbb{E}_{\tau} [\hat{G}^{AC}(\theta ; \tau)] \\
        &= \mathbb{E}_{\tau} [\sum_{k = 0}^H \gamma^k \nabla \log \pi_{\theta}(a_k|s_k) Q_w (s_k, a_k)]
    \end{aligned}
\end{equation*}
represent the expectation of the gradient estimator with respect to the sampled trajectory $\tau$ of length $H$. Let 
$$G_\infty(\theta) =\mathbb{E}_{\tau} [\sum_{k = 0}^\infty \gamma^k \nabla \log \pi_\theta(a_k|s_k) Q_w (s_k, a_k)] $$ 
represent the expectation of an infinite-horizon gradient estimator with respect to a sampled trajectory of infinite horizon. As before, we have the following noise-bias decomposition of the policy updates
\begin{equation*}
    \theta_{t+1} = \theta_t + \mu \nabla J (\theta_t) + \mu \xi_{t+1} + \mu d_{t+1},
\end{equation*}
where $\xi_{t+1} = \hat{G}^{AC} (\theta_t ; \tau_t) - G_H(\theta_t)$ once again represents a zero-mean noise term. Then we can further decompose the bias term $d_{t+1}$ as follows
\begin{equation*}
    \begin{aligned}
        d_{t+1} &= G_H (\theta_{t}) - \nabla J (\theta_{t}) \\
        &= G_H (\theta_{t}) - G_\infty(\theta_t) + G_\infty(\theta_t) - \nabla J (\theta_{t}) \\
        &= p_{t+1} + q_{t+1}
    \end{aligned}
\end{equation*}
where $p_{t+1} = G_H (\theta_{t}) - G_\infty(\theta_t)$ represents the bias component due to truncation of the infinite horizon and $q_{t+1} = G_\infty(\theta_t) - \nabla J (\theta_{t})$ represents the bias component induced by the critic approximation. In order to apply Theorem \ref{theorem:main}, we need to bound both $p_{t+1}$ and $q_{t+1}$; although the former can be bounded using the approach of the previous section, the latter requires novel techniques.

The structure of our algorithm is as follows. We consider a double-loop actor-critic algorithm with linear function approximation of $Q^{\pi_\theta}$ and an arbitrary policy parametrization. In the inner loop, the critic parameter $w$ is updated via TD(0) and projected onto a convex set $\Theta$, and in the outer loop, the policy parameter $\theta$ is updated via gradient updates. The gradient estimator algorithm is outlined in Algorithm \ref{alg:acgradientestimator}.

\begin{algorithm}
\caption{Actor-Critic Gradient Estimator}\label{alg:acgradientestimator}
\begin{algorithmic} 
\STATE \textbf{Input:} initial critic parameters $w_0$, policy parameter $\theta_t$, trajectory $\tau_t = \{s_0, a_0, s_1, a_1,... s_{H-1}, a_{H-1} \}$
\STATE Sample initial state-action pair $s'_0 \sim \rho_0$ and $a'_0 \sim \pi_{\theta_t}(\cdot | s'_0)$
\FOR{$k = 0, 1, 2, ... K-1$}
    \STATE Sample $s'_{k+1} \sim \mathcal{P}(\cdot | s'_k, a'_k)$ and $a'_{k+1} \sim \pi_{\theta_{t}}  (\cdot | s'_{k+1})$
    \STATE Compute the TD(0) semi-gradient $g_k(w_k)$
    \begin{equation*}
    \begin{aligned}
        g_k(w_k) =& (\mathcal{R}(s'_k, a'_k) + \gamma Q_{w_k}(s'_{k+1}, a'_{k + 1}) \\
        &-  Q_{w_k}(s'_k, a'_k)) \nabla Q_{w_k}(s'_k, a'_k) \\
    \end{aligned}
    \end{equation*}
    \STATE $w_{k+1} = Proj_{\Theta} [ w_{k} + \alpha_k g_k(w_k)]$

    \ENDFOR
    \STATE Calculate the averaged parameter $\bar{w}_{K,t}= \frac{1}{K}\sum_{k = 0}^{K-1} w_k$

\STATE Compute gradient estimator from $\tau_t$:
$$\hat{G}^{AC}(\theta_t ; \tau_t) = \sum_{h = 0}^{H-1} \gamma^h Q_{\bar{w}_{K,t}} (s_h, a_h) \nabla \log \pi_{\theta_{t}} (a_h | s_h)$$
\STATE \textbf{return} $\hat{G}^{AC}(\theta_t ; \tau_t)$
\end{algorithmic}
\end{algorithm}

To control the actor-critic bias $q_{t+1}$ and ensure second-order convergence by Theorem \ref{theorem:main}, we require that for each policy parameter iterate, the inner loop converges to the global optima $w^*(\theta_t)$ and that $Q_{w^*}(\theta_t)$ precisely approximates the true value function $Q^{\pi_{\theta_t}}$. The first requirement is satisfied by linear TD, as shown in Section \ref{sec:TDwoohoo}, and the second is formalized later in Assumption \ref{ass:zeroapprox}.

\subsection{Convergence of TD(0) on Nonstationary Markov Chains}
\label{sec:TDwoohoo}

The inner-loop structure of Algorithm \ref{alg:acgradientestimator} suggests that we should apply existing results regarding the finite-time convergence of TD(0) with Markovian sampling \cite{bhandariTD, liusplitting}. However, these analyses rely on an additional assumption that the Markov chain begins in the stationary distribution, which is unrealistic for the actor-critic setting since there is no guarantee that after each policy update we can begin at the new stationary distribution with respect to the updated policy. As argued in \cite{bhandariTD, liusplitting, dalalTD}, an exponentially mixing Markov chain approximately arrives at its stationary distribution after a logarithmic number of time steps. However, although this explanation justifies the assumption in a practical sense, it is not conducive for finite-time convergence analysis, since we can never reach the exact stationary distribution after a finite number of time steps.

As it turns out, the general convergence of TD(0) does hold without this initial state distribution assumption after some proof adjustments. In this section we reestablish the core results of \cite{bhandariTD} for a nonstationary Markov chain (i.e. a Markov chain that has not reached steady-state). These results, which are also utilized in \cite{qiuACzhaoran, liusplitting}, may be of independent interest.

\subsubsection{Setup}

We consider TD(0) with linear function approximation and a projection step. We define a set of $N$ feature functions $\phi_n : \mathcal{S} \times \mathcal{A} \to \mathbb{R}$, $ 0 < n \leq N$. For each state-action pair $(s, a)$, we define the vector $\phi(s, a) = (\phi_1(s, a), \phi_2(s, a), ..., \phi_N(s, a))^{\intercal}$ as the vector representing the features of $(s, a)$. Then we denote as $\Phi \in \mathbb{R}^{|\mathcal{S} \times \mathcal{A}| \times N}$ the matrix $\Phi = [\phi_1, ... \phi_N]$. Finally, we denote our linear parametrization of $Q^{\pi}$ by the parameter $w\in \mathbb{R}^N$ as  
$$Q_w (s, a) = w^{\intercal} \phi(s, a).$$
Let $g_k$ represent the stochastic semi-gradient at time step $k$ such that
\begin{equation*}
    \begin{aligned}
        g_k(w) =& \mathcal{R}(s_k, a_k)\phi(s_k, a_k) \\
        &+ (\gamma \phi(s_{k+1}, a_{k+1})^\intercal w - \phi (s_k, a_k)^\intercal w) \phi(s_k, a_k).
    \end{aligned}
\end{equation*}

As is common in the TD convergence literature, we consider TD(0) projected onto a convex set $\Theta$ that contains the limit point $w^*$ of the algorithm \cite{bhandariTD}. Therefore, the TD update can be written as
\begin{equation}
    w_{k+1} = Proj_{\Theta} [ w_{k} + \alpha g_k(w_k)].
\end{equation}

\subsubsection{Assumptions}

The following assumptions and Assumption \ref{ass_1boundedreward} are necessary to show the convergence of TD(0). 

\begin{assumption}
\label{ass:41nonzeropi}
    For all $\theta$ and $(s, a) \in \mathcal{S} \times \mathcal{A}$, we have $\pi_{\theta}(s, a)> 0$.
\end{assumption}

Assumption \ref{ass:41nonzeropi} is satisfied by popular policy parametrizations such as the softmax policy parametrization.

\begin{assumption}
\label{ass:ergodic}
For any $\pi > 0$, The Markov chain defined by $P (s, a, s', a') =  \mathcal{P}(s' | s, a) \pi(a'|s')$ is ergodic. 
\end{assumption}

Combined with Assumption \ref{ass:41nonzeropi}, Assumption \ref{ass:ergodic} is satisfied if there is a positive probability of transitioning between any two state-action pairs in a finite number of steps. Assumption \ref{ass:ergodic} is a common assumption in the TD and actor-critic literature \cite{bhandariTD, qiuACzhaoran, liusplitting, wuACquanquan, xuAC, xuAC2020nestedloop} that has a number of implications; the Markov chain is irreducible and aperiodic, it has a unique stationary distribution $\eta_{\pi}(s, a)$, and $\eta_{\pi}(s, a) \neq 0$ for all $(s, a)$. In addition, the Markov chain mixes at a uniform geometric rate, i.e., there exists $m > 0$, $r \in (0, 1)$,   such that for $t \in \mathbb{N}$ we have
\begin{equation}
   \sup_{\substack{s \in \mathcal{S} \\ a \in \mathcal{A}}}d_{TV} (\mathbb{P}(s_t = \cdot, a_t = \cdot| s_0 = s, a_0 
= a ), \eta_{\pi}) \leq m r^t
\end{equation}
and $\tau^{\text{\tiny mix}}(\epsilon) = \min \{ t \in \mathbb{N} \, | \, m r^t \leq \epsilon \}$ is the mixing time.

\begin{assumption}
\label{ass:finite}
    The number of states and actions is finite.
\end{assumption}
As discussed in \cite{bhandariTD}, Assumption \ref{ass:finite} can be relaxed to consider countably infinite state-action pairs.

\begin{assumption}
\label{ass:fullrank}
The feature matrix $\Phi$ has full column rank, i.e., the feature vectors $\{ \phi_1,...\phi_N \}$ are linearly independent. In addition, for all $(s, a) \in \mathcal{S} \times \mathcal{A}$, $\lVert \phi(s, a)\lVert ^2 \leq 1$ .
\end{assumption}

\begin{assumption}
\label{ass:boundR}
There exists $R > 0$ such that  
$diam(\Theta)~\leq~R,$
where $diam$ is the diameter.
    
\end{assumption}

Assumptions \ref{ass:fullrank} and \ref{ass:boundR} are also common in the literature \cite{bhandariTD, liusplitting, tsitsiklisTD, qiuACzhaoran}. See Section 8.2 of \cite{bhandariTD} and Proposition 3 of \cite{qiuACzhaoran} for additional details on defining $R$.

Finally, let $A_\theta$ denote the positive definite matrix $\mathbb{E} _ {(s, a) \sim \eta_\theta, (s', a') \sim P(s, a, \cdot)}[\phi(s, a) (\phi(s, a) - \gamma \phi(s', a'))^{\intercal}]$. Based on Assumptions \ref{ass:ergodic} and \ref{ass:fullrank}, we can conclude that $A_\theta$ is positive definite \cite{tsitsiklisTD}. We further assume the smallest eigenvalue of $A_\theta$ is uniformly bounded away from zero in the following assumption, which is also utilized in \cite{qiuACzhaoran}. 

\begin{assumption}
\label{ass:boundedeigenvalue}
There exists a lower bound $\varsigma > 0$, such that for all $\theta \in \mathbb{R}^M$ we have
$\lambda_{min}(A_\theta + A_\theta^\intercal) \geq \varsigma.$
\end{assumption}

\subsubsection{Finite-Time Convergence of TD(0)}

We first require the following theorem from \cite{tsitsiklisTD}, which establishes the existence and uniqueness of the solution to the projected Bellman equation as the limit point of the TD(0) algorithm.

\begin{theorem} 
\label{theorem:existunique}
\cite{tsitsiklisTD} Denote by $\mathbb{T}^{\pi}$ the Bellman operator under policy $\pi$ such that for a value function $Q : \mathcal{S} \times \mathcal{A} \to \mathbb{R}$ we have
\begin{equation*}
\begin{aligned}
    (\mathbb{T}^\pi Q) (s, a) =& \mathcal{R}(s, a) + \gamma \sum_{\substack{s' \in \mathcal{S} \\ a' \in \mathcal{A}}} \mathcal{P}(s' | s, a) \pi(a' | s') Q(s', a').
\end{aligned}
\end{equation*}

Then given Assumptions \ref{ass_1boundedreward}, \ref{ass:41nonzeropi}-\ref{ass:boundR}, the limit point $w^*$ of the TD(0) algorithm with linear function approximation exists, and it is the unique solution to the projected Bellman equation
$$\Phi w = Proj_{\Phi} (\mathbb{T}^{\pi}\Phi w)$$
where $Proj_{\Phi}$ is the projection operator onto the subspace $\{ \Phi x | x \in \mathbb{R}^N\}$ spanned by the feature vectors $\phi_n$.

\end{theorem}

Now we share our main result regarding the convergence of TD(0). Theorem \ref{theorem:bhandarialt} establishes the convergence rate of the Projected TD(0) algorithm after $T$ constant time steps on a nonstationary Markov chain.

\begin{theorem}
\label{theorem:bhandarialt}
Suppose Assumptions \ref{ass_1boundedreward}, \ref{ass:41nonzeropi}-\ref{ass:boundR} hold and $\bar{w}_K=\frac{1}{K}\sum_{k = 0}^{K-1} w_k$ is generated by $K$ steps of the Projected TD(0) algorithm with $w^* \in \Theta$ and $\alpha = \frac{1}{\sqrt{K}}$. Then
\begin{equation*}
\begin{aligned}
    \mathbb{E}[\lVert Q_{w^*} - Q_{\bar{w}_K}\lVert ^2_{\eta_{\pi}}] \leq& \frac{\lVert w^* - w_0\lVert ^2 + F^2 (17 + 12\tau^{\text{\tiny mix}}(\frac{1}{\sqrt{K}}))}{2 (1 - \gamma) \sqrt{K}} \\
    &+ \frac{10 F^2 m}{(1 - r)(1 - \gamma) K},
\end{aligned}
\end{equation*}
where $F = \mathcal{R}_{max} + 2R$ and 
$$\lVert Q_{w^*} - Q_{\bar{w}_K}\lVert ^2_{\eta_{\pi}} = \sum_{\substack{s \in \mathcal{S} \\ a \in \mathcal{A}}} \eta_{\pi}(s, a) (Q_{w^*}(s, a) - Q_{\bar{w}_K}(s, a))^2.$$
\end{theorem}

\textit{Proof.} See Appendix \ref{appendix:td(0)}.

We can compare this result with Theorem 3 from \cite{bhandariTD} which shows
\begin{equation}
\mathbb{E}[\lVert Q_{w^*} - Q_{\bar{w}_K}\lVert ^2_{\eta_{\pi}}] \leq \frac{\lVert w^* - w_0\lVert ^2 + F^2 (9 + 12\tau^{\text{\tiny mix}}(\frac{1}{\sqrt{K}}))}{2 (1 - \gamma) \sqrt{K}} 
\end{equation}
for TD(0) on a stationary Markov chain. We note that the nonstationary result involves slightly different constants and an additional term that decays at the rate of $O(1/K)$.

\subsection{Main Result}

Returning to the actor-critic setting, the following assumption combined with the existence and uniqueness theorem (Theorem \ref{theorem:existunique}) implies for the limit point of the TD(0) algorithm $w^*$, the resulting function $Q_{w^*} = {w^*}^T \phi(s, a)$ approximates the true state-action value function $Q^{\pi_\theta}$ with arbitrarily close precision.

\begin{assumption}
\label{ass:zeroapprox}
     The value function lies in the linear function class such that
\begin{equation*}
    \inf_{w \in \Theta} \mathbb{E}_{(s, a) \sim \eta_{\pi_\theta}}[(\mathbb{T}^{\pi_\theta} (w^{\intercal}\phi(s,a)) - w^{\intercal}\phi(s,a))^2] = 0.
\end{equation*}
\end{assumption}

Assumption \ref{ass:zeroapprox} is a linear realizability assumption that states that the value function can be sufficiently represented by a linear model of the feature vectors. This can be satisfied via an appropriate choice of feature vectors, such as radial basis functions, Fourier basis functions, or neural networks \cite{ziweiapproximation}. Other actor-critic works \cite{xuAC, kumar, FuAC} also require similar assumptions.


By characterizing the convergence of TD(0), we show in the following lemma that the norm of the bias term $q_{t+1}$ decays with respect to the number of inner-loop iterations. Lemma \ref{lemma:FAbiasboundAC} features diminishing step sizes because it requires a stronger fourth moment bound showing the direct convergence of $\bar{w}_K$ to $w^*$, as opposed to the weaker result obtained in Theorem \ref{theorem:bhandarialt}. The details of the proof are in Appendix \ref{sec:FAbiasboundAC}.

\begin{lemma}
\label{lemma:FAbiasboundAC}
 For $K = O(\frac{\log^2(\mu^{-4})}{\mu^4})$ as $\mu \to 0$, for which $O(\cdot)$ does not hide dependencies on other constants, and
$$D_q = G \left(\frac{192 F^2 R^2}{\varsigma^2 \log^2(r^{-1})} + O(\frac{1}{\log \mu^{-4}}) \right)^{-1/4}$$
for which $O(\cdot)$ hides dependencies on $F$, $r$, $\varsigma$, $m$, the expected function approximation bias is bounded such that
$$\mathbb{E}[\lVert q_{t+1}\lVert ^4 | \mathcal{F}_t] \leq G^4\mathbb{E}[ \lVert \bar{w}_{K,t} - w^*\lVert ^4 | \mathcal{F}_t] \leq D_q^4\mu^4 .$$
\end{lemma}

Finally, by combining Theorem \ref{theorem:main} with Lemma \ref{lemma:FAbiasboundAC} and bounds on the truncation bias $p_{t+1}$ (Lemma \ref{lemma:truncboundAC}) and noise $\xi_{t+1}$ (Lemma \ref{lemma:ACboundedbiasnoise}), we arrive at Theorem \ref{theorem:acepsilon}. 

\begin{theorem}
\label{theorem:acepsilon}
 Suppose Assumptions \ref{ass_1boundedreward}-\ref{ass_6noisecurv}, \ref{ass:41nonzeropi}-\ref{ass:boundedeigenvalue}, \ref{ass:zeroapprox} hold and let $\epsilon > 0$. For $\mu~<~\frac{\epsilon^2 \delta}{L\sigma^2 + D^2}$ where $\sigma = \frac{GR}{1 - \gamma}$,  $D = 2 (D_p^4 + D_q^4)^{1/4}$ and $D_p = \frac{GR}{1 - \gamma}$, we have with probability $1 - \delta$ that Algorithm 1 with actor-critic policy gradient estimator computed via Algorithm \ref{alg:acgradientestimator} with $H = O(\log(\epsilon^{-2}))$ and $K = \tilde{O}(\epsilon^{-8})$
reaches an $\epsilon$-second order stationary point in $\tilde{O}(\epsilon^{-6.5})$ iterations.
\end{theorem}
\textit{Proof. }See Appendix \ref{sec:appendACnoisebias}.

In Theorem \ref{theorem:acepsilon}, $O(\cdot)$ hides dependency on $\gamma$ and $\tilde{O}(\cdot)$ hides dependencies on $L$, $G$, $R$, $\gamma$, $M$, $\sigma_l$, $\chi$, $\delta$, $F$, $r$, $m$, $\varsigma$. 

\section{Conclusion}

In this work, we provide a novel analysis on the convergence of biased policy gradient methods to second-order stationary points. Our work applies to general policy parametrization and Markovian sampling. We also show the convergence of TD(0) on nonstationary Markov chains, which pertains to realistic actor-critic implementations.

Future directions may involve extending this work to two-timescale or single-timescale actor-critic algorithms, which may provide some performance improvement. In addition, instead of assuming Assumption \ref{ass:zeroapprox}, we may want to show second-order convergence of actor-critic algorithms with respect to some irremoveable approximation error $\epsilon_{app}$ representing the imperfect critic approximation, similar to several first-order analyses \cite{wuACquanquan, qiuACzhaoran}.

\section*{Impact Statement}

This paper advances the field of reinforcement learning by characterizing the solutions achieved by policy gradient algorithms. There are many potential societal consequences of our work, none which we feel must be specifically highlighted here.

\bibliography{example_paper}
\bibliographystyle{icml2024}

\newpage
\appendix
\onecolumn

\section{Additional Discussion of Related Work}
\label{app:morediscussion}

\subsection{Global Convergence}
\label{app:global}
We first present some of the many global convergence results available in terms of $\epsilon$-optimality. In reinforcement learning, global convergence guarantees for policy gradient algorithms tend to arise from some underlying structure of the optimization problem, due to specific policy parametrization or algorithm.

For first-order algorithms, the following papers achieve global optimality.
\begin{itemize}
    \item In \cite{agarwal}, tabular parametrization with exact gradients, $O(\epsilon^{-2})$ iterations
    \item In \cite{bhandariglobalPGforreal}, objective functions that satisfy a gradient dominance property with exact gradients, $O(\epsilon^{-2})$ iterations
    \item In \cite{pmlr-v119-mei20b}, tabular softmax parametrization with exact gradients, $O(\epsilon^{-1})$ iterations
    \item In \cite{wangzhaoranNeuralPG}, neural policy gradient with extremely wide shallow neural networks, $O(\epsilon^{-2})$ iterations
\end{itemize}

More powerful global convergence results are achieved for quasi-second-order methods like natural policy gradient and mirror descent policy gradient. These results can go beyond the tabular setting.
\begin{itemize}
    \item In \cite{agarwal}, natural policy gradient with tabular softmax parametrization with exact gradients, $O(\epsilon^{-1})$ iterations
    \item In \cite{xuAC2020nestedloop}, natural actor-critic with function approximation where the feature vectors vary in each iteration, $O(\epsilon^{-4})$ outer-loop iterations
    \item In \cite{cayci2022finitetime}, natural actor-critic policy gradient where the actor and critic are parametrized with extremely wide neural networks, $O(\epsilon^{-2})$ iterations
    \item In \cite{yuan2023linear}, natural policy gradient for log-linear policies with compatible function approximation, linear convergence in terms of outer loop iterations
    \item In \cite{alfano2024mirrordescent}, policy mirror descent with general parametrization, linear convergence in terms of outer loop iterations
\end{itemize}
In comparison, we achieve an $\epsilon$-second-order stationary point in $\tilde{O}(\epsilon^{-6.5})$ iterations. Our work focuses on first-order algorithms and general policy parametrization which are simpler to implement and more widely used in practice. We also do not use oracles or exact gradients in our analysis. Finally, our work also has no dependence on the distribution mismatch coefficient, which is widely used in global convergence results and roughly quantifies how well the initial state distribution matches the optimal state distribution.

\subsection{Second-Order Convergence}
\label{app:secondorder}
Our work inherits the sample complexity of $\tilde{O}(\epsilon^{-6.5})$ from \cite{vlaski2019secondorder}, which is weaker than the best sample complexity of $\tilde{O}(\epsilon^{-4})$ obtained by \cite{jinsaddleSGD}, both of which analyze the second-order convergence of vanilla stochastic gradient descent. Faster second-order convergence is available for algorithms with exact gradients, such as perturbed gradient descent which converges in $\tilde{O}(\epsilon^{-2})$ iterations  \cite{Jin2017howtoescape}. However, exact gradient computations are intractable for realistic policy gradient implementations. Second-order or quasi-second-order algorithms that utilize Hessian information can also converge faster. SPIDER-SFO from \cite{fangspider} obtains a sample complexity of $\tilde{O}(\epsilon^{-3})$ via a negative curvature search method. Similar techniques are extended to the policy gradient setting in \cite{khorasani2023efficiently} with complexity $\tilde{O}(\epsilon^{-3})$ and \cite{WANG2022109687} with complexity $\tilde{O}(\epsilon^{-3.5})$. All of these aforementioned works only deal with unbiased gradient estimators. To the best of our knowledge only \cite{vlaskiSOFL} show second-order convergence with biased gradient estimators in the form of federated learning, requiring $\tilde{O}(\epsilon^{-6.5})$ global iterations. This matches our sample complexity, as expected.

\section{Proof of Theorem \ref{theorem:main}}
\label{theorem2}
\subsection{Key Lemmas and Proof Sketch}
Our approach for proving that Algorithm 1 arrives at an $\epsilon$-second order stationary point relies on bounding the bias and noise of the gradient estimator and applying the techniques developed in \cite{vlaski2019secondorder}. Broadly speaking, \cite{vlaski2019secondorder} show second-order convergence for unbiased stochastic gradient descent by first showing the iterates on the second-order Taylor expansion of the objective function escape saddle points, and then showing that the iterates and those on its Taylor approximation are sufficiently close. Therefore, if we also show that the policy gradient iterates are close to the iterates on the Taylor expansion, we can conveniently apply the convergence results of \cite{vlaski2019secondorder} to our setting.

We begin with establishing fourth moment bounds on the noise term $\xi_{t+1}$ by the following lemma.

\begin{lemma}
\label{lemma:sigma}
 Suppose for some random variable $X$ we have $\mathbb{E}[X] = \mu$ and $\lVert X \rVert \leq \sigma$. Then 
\begin{equation*}
    \mathbb{E}[\lVert X - \mu\lVert ^2] \leq \sigma^2, \\
\end{equation*}
\begin{equation*}
    \mathbb{E}[\lVert X - \mu\lVert ^4] \leq 4 \sigma^2.
\end{equation*}
\end{lemma} 

\begin{proof}
\begin{equation*}
        \mathbb{E}[\lVert X - \mu\lVert ^2] = \mathbb{E}[\lVert X\lVert ^2] - \mu^2 \leq \mathbb{E}[\lVert X\lVert ^2] \leq \sigma^2
\end{equation*}
\begin{equation*}
\begin{aligned}
    \lVert X - \mu\lVert ^4 \leq \lVert X - \mu\lVert ^2 \cdot 4 \sigma^2 \\
    \mathbb{E}[\lVert X - \mu\lVert ^4 ] \leq 4 \sigma^4
\end{aligned}
\end{equation*}
\end{proof}
So by Lemma \ref{lemma:sigma}, we have
\begin{equation}
\label{eq:noisebound2}
    \mathbb{E}[\lVert \xi_{t+1}\lVert ^2 | \mathcal{F}_t] \leq \sigma^2,
\end{equation}
\begin{equation}
\label{eq:noisebound4}
    \mathbb{E}[\lVert \xi_{t+1}\lVert ^4 | \mathcal{F}_t] \leq 4\sigma^4.
\end{equation}
In addition, by Jensen's inequality, the fourth-moment bound 
$$\mathbb{E}[\lVert d_{t+1}\lVert ^4 | \mathcal{F}_t] \leq D^4 \mu^4$$
also implies the following second-moment bound 
\begin{equation}
\label{eq:momentbound2}
    \mathbb{E}[\lVert d_{t+1}\lVert ^2 | \mathcal{F}_t] \leq D^2 \mu^2.
\end{equation}

To proceed with the proof, we first show that in the large gradient regime, we observe a large ascent in function value, whereas around local maxima, the possible descent is bounded. Then we construct a pair of coupled sequences $\{\theta_{i+j}\}$ and $\{\theta'_{i+j}\}$, where $\{\theta_{i+j}\}$ represents the gradient iterates on the original objective function and $\{\theta'_{i+j}\}$ represents gradient ascent iterates on the second-order Taylor approximation of the function centered at $\theta_i$ with the same noise term. Through moment bounds, we show that the difference between the coupled sequences is sufficiently small. These results allow us to leverage Theorems 2 and 3 in \cite{vlaski2019secondorder}.

The following lemma establishes that for small enough step sizes, we have sufficient ascent starting in the large gradient regime $\theta_i \in \mathcal{G}$, and descent is bounded starting near a local maxima $\theta_i \in \mathcal{M}$. 

\begin{lemma}
\label{lemma:suffascent}
For $\mu < \frac{1}{L}$, we have after one iteration of Algorithm \ref{alg:pseudo},
$$\mathbb{E}[J(\theta_{i+1}) | \theta_i \in \mathcal{G}] \geq \mathbb{E}[J(\theta_i) | \theta_i \in \mathcal{G}] + \frac{\mu^2 (L\sigma^2 + D^2 \mu)}{2 \delta} $$
$$\mathbb{E}[J(\theta_{i+1}) | \theta_i \in \mathcal{M}] \geq \mathbb{E}[J(\theta_i) | \theta_i \in \mathcal{M}]  - \frac{\mu^2 (L\sigma^2 + D^2 \mu)}{2} .$$
\end{lemma} 

\textit{Proof of Lemma \ref{lemma:suffascent}: } see Appendix \ref{lemma4}.


Beginning from $\theta_i \in \mathcal{H}$, we define $\{\theta'_{i + j} \}$ as the gradient ascent iterates on the second-order Taylor approximation of $J(\theta)$ plus the noise term $\xi_{i + j + 1}$ from the original sequence.
Denote the Taylor expansion around $\theta_i$ as $\hat{J}$ as follows
    $$\hat{J}(\theta) = J(\theta_i) + \nabla J(\theta_i)^T (\theta - \theta_i) + \frac{1}{2} (\theta - \theta_i)^\intercal \nabla^2 J(\theta_i) (\theta - \theta_i)$$
    $$\nabla \hat{J}(\theta) =  \nabla J(\theta_i) + \nabla^2 J(\theta_i) (\theta - \theta_i).$$
    So we have $\{\theta'_{i + j} \}$ and our original sequence iterates $\{\theta_i\}$ defined as follows
    $$\theta'_{i + j + 1} = \theta'_{i + j} + \mu \nabla J(\theta_{i}) + \mu \nabla^2 J(\theta_i) (\theta'_{i + j} - \theta_i) + \mu \xi_{i + j + 1}$$
    $$\theta_{i+j+1} = \theta_{i+j} + \mu \nabla J (\theta_{i+j}) + \mu \xi_{i+j+1} + \mu d_{i+j+1}.$$
Then we can conclude in the following lemma that in the vicinity of a saddle point, the distance between $\theta_i$ and $\theta_{i+j+1}$ is bounded, and the distance between $\theta_i$ and  $\theta'_{i}$ is bounded.

\begin{lemma}
\label{lemma:bounds}
For $\{\theta_i\}$ and $\{\theta'_i\}$ defined above, and $j \leq \frac{C}{\mu}$, where $C$ is a constant independent of $\mu$, we have
\begin{equation}
\label{eq:mainbound1}
\mathbb{E}[\lVert\theta_i - \theta_{i+j+1}\lVert^2 | \theta_i \in \mathcal{H}] \leq O(\mu)
\end{equation}
\begin{equation}
\label{eq:mainbound2}
\mathbb{E}[\lVert\theta_i - \theta_{i+j+1}\lVert^4 | \theta_i \in \mathcal{H}] \leq O(\mu^2)
\end{equation}
\begin{equation}
\label{eq:mainbound3}
\mathbb{E}[\lVert\theta'_{i+j+1} - \theta_{i+j+1}\lVert^2 | \theta_i \in \mathcal{H}] \leq O(\mu^2).
\end{equation}
\end{lemma}
\begin{corollary}
From the results of Lemma \ref{lemma:bounds}, we can conclude
\begin{equation}
    \mathbb{E}[\lVert\theta_i - \theta_{i+j+1}\lVert^3 | \theta_i \in \mathcal{H}] \leq O(\mu^{3/2})
\end{equation}
\begin{equation}
    \mathbb{E}[\lVert\theta_i - \theta'_{i+j+1}\lVert^2 | \theta_i \in \mathcal{H}] \leq O(\mu)
\end{equation}
\begin{equation}
    \mathbb{E}[\lVert\theta_i - \theta'_{i+j+1}\lVert^3 | \theta_i \in \mathcal{H}] \leq O(\mu^{3/2}).
\end{equation}
The first inequality follows from Jensen's inequality, and the second and third follow from the bounds on $\lVert\theta_i - \theta_{i+j+1}\lVert$ and $\lVert\theta'_{i+j+1} - \theta_{i+j+1}\lVert$.
\end{corollary}

\textit{Proof of Lemma \ref{lemma:bounds}.} See Appendix \ref{lemma5}.

\subsection{Proof of Theorem \ref{theorem:main}}
\begin{proof}
For the sequences $\{\theta_i\}$ and $\{\theta'_i\}$ defined above, suppose the moment bounds in Lemma \ref{lemma:bounds} hold. Then from Corollary 1 in \cite{vlaski2019secondorder},  beginning at $\theta_i \in \mathcal{H}$ for the finite horizon $j \leq \frac{C}{\mu}$ we have
$$\mathbb{E}[ J(\theta_{i+j}) | \theta_i \in \mathcal{H} ] \geq \mathbb{E}[J(\theta'_{i+j}) | \theta_i \in \mathcal{H} ] - O(\mu^{3/2}), $$
which basically states that the function values on $\{\theta_i \}$ stay close to the function values on $\{\theta'_i\}$. This allows us to conclude that sufficient ascent occurs on the Taylor approximation as well as the original function by way of Theorem 2 from \cite{vlaski2019secondorder}. Beginning at a strict saddle point $\theta_i \in \mathcal{H}$, gradient ascent iterates on the short-term model for $\mathcal{T}$ iterations after $i$ with 
$$\mathcal{T} = \frac{\log (2 M \frac{\sigma^2}{\sigma^2_l} + 1)}{\log(1 + 2 \mu \omega)} \leq O(\frac{1}{\mu \omega})$$
guarantees 
$$\mathbb{E}[J(\theta'_{i + \mathcal{T}}) | \theta_i \in \mathcal{H} ] \geq \mathbb{E} [ J(\theta_i) | \theta_i \in \mathcal{H} ] + \frac{\mu}{2} M \sigma^2 - o(\mu).$$
Combined with the bounds on the iterates from Lemma \ref{lemma:bounds}, this implies
$$\mathbb{E}[J(\theta_{i + \mathcal{T}}) | \theta_i \in \mathcal{H} ] \geq \mathbb{E} [ J(\theta_i) | \theta_i \in \mathcal{H} ] + \frac{\mu}{2} M \sigma^2 - o(\mu).$$
This result, in combination with Lemma \ref{lemma:suffascent}, and the observation that $|J(\theta)| \leq \frac{\mathcal{R}_{\max}}{1 - \gamma}$ for all $\theta$ allows us to apply Theorem 3 from \cite{vlaski2019secondorder},
yielding our final result.
\end{proof}

\subsection{Proof of Lemma \ref{lemma:suffascent}}
\label{lemma4}
\begin{proof}
Our iterates are
$$\theta_{i+1} = \theta_i + \mu \nabla J (\theta_i) + \mu \xi_{i+1} + \mu d_{i+1}.$$
Because $J$ is Lipschitz smooth by Lemma \ref{lemma_lipschitz}, we have
\begin{equation*}
    \begin{aligned}
        J(\theta_{i+1}) \geq& J(\theta_i) + \nabla J (\theta_i)^T (\theta_{i+1} - \theta_i) - \frac{L}{2} \lVert\theta_{i+1} - \theta_i\lVert^2 \\
        \geq& J(\theta_i) + \mu \nabla J (\theta_i)^T (\nabla J(\theta_i) + \xi_{i+1} + d_{i+1}) - \frac{L \mu^2}{2} \lVert\nabla J(\theta_i) + \xi_{i+1} + d_{i+1}\lVert^2 \\
        \geq& J(\theta_i) + \mu \lVert\nabla J (\theta_i)\lVert^2 + \mu \nabla J (\theta_i)^T \xi_{i+1} + \mu \nabla J (\theta_i)^T d_{i+1} \\
        &- \frac{L \mu^2}{2} (\lVert\nabla J(\theta_i) + d_{i+1}\lVert^2 + \lVert \xi_{i+1}\lVert^2 + 2(\nabla J(\theta_i) + d_{i+1})^T \xi_{i+1}).
    \end{aligned}
\end{equation*}
We can take expectation with respect to the filtration $\mathcal{F}_i$ on either side to remove the cross terms with the noise term $\xi_{i+1}$, and then we have by (\ref{eq:noisebound2})
\begin{equation*}
    \begin{aligned}
        \mathbb{E}[J(\theta_{i+1}) | \mathcal{F}_i] &\geq J(\theta_i) + \mu \lVert\nabla J (\theta_i)\lVert^2  + \mu \mathbb{E}[\nabla J (\theta_i)^T d_{i+1} | \mathcal{F}_i] - \frac{L \mu^2}{2} \mathbb{E} [\lVert\nabla J(\theta_i) + d_{i+1}\rVert^2 | \mathcal{F}_i] - \frac{L \mu^2}{2} \mathbb{E}[\lVert \xi_{i+1}\lVert^2 | \mathcal{F}_i] \\
        &\geq J(\theta_i) + \mu \lVert\nabla J (\theta_i)\lVert^2  + \mu \mathbb{E}[\nabla J (\theta_i)^T d_{i+1}  | \mathcal{F}_i] - \frac{L \mu^2}{2} \mathbb{E}[\lVert\nabla J(\theta_i) + d_{i+1}\rVert^2  | \mathcal{F}_i] - \frac{L \mu^2 \sigma^2 }{2} .
    \end{aligned}
\end{equation*}
We assume that $\mu < \frac{1}{L}$ to obtain
$$\mathbb{E}[J(\theta_{i+1}) | \mathcal{F}_i] \geq J(\theta_i) + \mu \lVert\nabla J (\theta_i)\lVert^2  + \mu \mathbb{E}[\nabla J (\theta_i)^T d_{i+1} | \mathcal{F}_i]  - \frac{\mu}{2} \mathbb{E}[\lVert\nabla J(\theta_i) + d_{i+1}\rVert^2  | \mathcal{F}_i] - \frac{L \mu^2 \sigma^2 }{2} .$$
Then we use the fact that $\lVert a + b\lVert^2 = \lVert a\lVert^2 + 2 a^T b + \lVert b\lVert^2$ and (\ref{eq:momentbound2}) to obtain
\begin{equation*}
    \begin{aligned}
        \mathbb{E}[J(\theta_{i+1}) | \mathcal{F}_i] \geq& J(\theta_i) + \mu \lVert\nabla J (\theta_i)\lVert^2  + 
        \mu \mathbb{E}[\nabla J (\theta_i)^T d_{i+1} | \mathcal{F}_i]  - \frac{\mu}{2} \lVert\nabla J(\theta_i)\lVert^2 - \frac{\mu}{2} \mathbb{E}[\lVert d_{i+1}\lVert^2  | \mathcal{F}_i]  \\
        &- \mu \mathbb{E}[ \nabla J(\theta_i)^T d_{i+1} | \mathcal{F}_i] - \frac{L \mu^2 \sigma^2 }{2}  \\
        =& J(\theta_i) + \frac{\mu}{2} \lVert \nabla J (\theta_i)\lVert^2 - \frac{\mu}{2} \mathbb{E}[\lVert d_{i+1}\lVert^2 | \mathcal{F}_i] - \frac{L \mu^2 \sigma^2 }{2} \\
        \geq& J(\theta_i) + \frac{\mu}{2} \lVert\nabla J (\theta_i)\lVert^2 - \frac{D^2 \mu^3}{2} - \frac{L \mu^2 \sigma^2 }{2} .
    \end{aligned}
\end{equation*}
Now we want to apply the law of total expectation and condition on where $w_i$ is located in the parameter space. We first condition on $\theta_i \in \mathcal{G}$, where we have $\lVert\nabla J(\theta_i)\lVert^2 > \mu (L \sigma^2 + D^2 \mu) (1 + \frac{1}{\delta})$ to arrive at
$$\mathbb{E}[J(\theta_{i+1}) | \theta_i \in \mathcal{G}] \geq \mathbb{E}[J(\theta_i) | \theta_i \in \mathcal{G}] + \frac{\mu}{2} \mu (L \sigma^2 + D^2 \mu) (1 + \frac{1}{\delta}) - \frac{L \mu^2 \sigma^2 }{2} - \frac{D^2 \mu^3}{2} $$
$$\mathbb{E}[J(\theta_{i+1}) | \theta_i \in \mathcal{G}] \geq \mathbb{E}[J(\theta_i) | \theta_i \in \mathcal{G}] + \frac{\mu^2 (L\sigma^2 + D^2 \mu)}{2 \delta} .$$
If we instead condition on $\theta_i \in \mathcal{M}$, we have that 
$$\mathbb{E}[J(\theta_{i+1}) | \theta_i \in \mathcal{M}] \geq \mathbb{E}[J(\theta_i) | \theta_i \in \mathcal{M}]  - \frac{\mu^2 (L\sigma^2 + D^2 \mu)}{2} .$$
\end{proof}

\subsection{Proof of Lemma \ref{lemma:bounds}}
\label{lemma5}

Before we proceed with the proof of Lemma \ref{lemma:bounds}, we require a preliminary lemma from \cite{vlaski2019secondorder} that will help us show that our product does not blow up for small $\mu$.

\begin{lemma}
\label{lemma:minilemma}
For $C, \mu, L > 0$ and $k \in \mathbb{Z}_{+}$ with $\mu < \frac{1}{L}$
$$\lim_{\mu \to 0} \left(\frac{( 1 + \mu L)^k + O(\mu^2)}{(1 - \mu L)^{k-1}} \right) ^{C/\mu} = O(1). $$
\end{lemma}

\textit{Proof of Lemma \ref{lemma:minilemma}.} See \cite{vlaski2019secondorder}.
\begin{proof}
First we want to show (\ref{eq:mainbound1}), restated below
$$\mathbb{E}[\lVert\theta_i - \theta_{i+j+1}\lVert^2 | \theta_i \in \mathcal{H}] \leq O(\mu).$$
We have by (\ref{eq:noisebound2})
\begin{equation*}
    \begin{aligned}
        \lVert\theta_i - \theta_{i+j+1}\rVert^2  &= \lVert\theta_i - \theta_{i+j} - \mu \nabla J(\theta_{i + j}) - \mu \xi_{i+j+1} - \mu d_{i+j+1}\lVert^2 \\
        \mathbb{E}[\lVert\theta_i - \theta_{i+j+1}\lVert^2 | \mathcal{F}_{i+j}]  &= \mathbb{E}[\lVert\theta_i - \theta_{i+j} - \mu \nabla J(\theta_{i + j}) - \mu d_{i+j+1}\lVert^2 | \mathcal{F}_{i+j}] + \mu^2 \mathbb{E}[\lVert \xi_{i+j+1}\lVert^2 | \mathcal{F}_{i+j}] \\
        &\leq \mathbb{E}[\lVert\theta_i - \theta_{i+j} - \mu \nabla J(\theta_{i + j}) - \mu d_{i+j+1}\rVert^2 | \mathcal{F}_{i+j}]  + \mu^2 \sigma^2 \\
        &= \mathbb{E}[\lVert\theta_i - \theta_{i+j} - \mu \nabla J(\theta_{i + j}) + \mu \nabla J(\theta_i) - \mu \nabla J(\theta_i) - \mu d_{i+j+1}\rVert^2 | \mathcal{F}_{i+j}] + \mu^2 \sigma^2.
    \end{aligned}
\end{equation*}
By Jensen's inequality, we have for $0 < \alpha < 1$,
$$\lVert a + b\lVert^2 \leq \frac{1}{\alpha} \lVert a\lVert^2 + \frac{1}{1 - \alpha} \lVert b\lVert^2.$$
So we have
\begin{equation}
\label{eq:bound1}
    \mathbb{E}[\lVert\theta_i - \theta_{i+j+1}\lVert^2 | \mathcal{F}_{i+j}] \leq \frac{1}{1 - \mu L}\lVert\theta_i - \theta_{i+j} - \mu \nabla J(\theta_{i + j}) + \mu \nabla J(\theta_i)\lVert^2 + \frac{\mu^2}{\mu L} \mathbb{E}[ \lVert\nabla J(\theta_i) + d_{i+j+1}\lVert^2  | \mathcal{F}_{i+j}] + \mu^2 \sigma^2.
\end{equation}
We consider the first term on the right hand side of (\ref{eq:bound1}) and expand it to obtain
$$\lVert\theta_i - \theta_{i+j} - \mu \nabla J(\theta_{i + j}) + \mu \nabla J(\theta_i)\lVert^2 $$
$$\leq \lVert\theta_i - \theta_{i+j}\lVert^2 + 2\mu \lVert\theta_i - \theta_{i+j}\lVert\cdot\lVert\nabla J(\theta_{i + j}) - \nabla J(\theta_i)\lVert +  \mu^2 \lVert\nabla J(\theta_{i + j}) - \nabla J(\theta_i)\lVert^2.$$
By Lipschitz smoothness, we have
\begin{equation*}
\begin{aligned}
    \lVert\theta_i - \theta_{i+j} - \mu \nabla J(\theta_{i + j}) + \mu \nabla J(\theta_i)\lVert^2 &\leq \lVert\theta_i - \theta_{i+j}\lVert^2 + 2\mu L \lVert\theta_i - \theta_{i+j}\lVert\cdot\lVert\theta_{i + j} - \theta_i\lVert +  \mu^2 L^2 \lVert\theta_{i + j} - \theta_i\lVert^2 \\
    &= (1 + 2 \mu L + \mu^2 L^2)\lVert\theta_i - \theta_{i+j}\lVert^2 \\
    &= (1 + \mu L)^2\lVert\theta_i - \theta_{i+j}\lVert^2.
\end{aligned}
\end{equation*}
We plug this into our original expression (\ref{eq:bound1}) and use (\ref{eq:momentbound2}) to obtain  
\begin{equation*}
    \begin{aligned}
        \mathbb{E}[\lVert\theta_i - \theta_{i+j+1}\lVert^2 | \mathcal{F}_{i+j}] 
 &\leq \frac{(1 + \mu L)^2}{1 - \mu L}\lVert\theta_i - \theta_{i+j} \lVert^2 + \frac{\mu}{ L}\mathbb{E}[ \lVert\nabla J(\theta_i) + d_{i+j+1}\lVert^2 | \mathcal{F}_{i+j}] + \mu^2 \sigma^2 \\
 &\leq \frac{(1 + \mu L)^2}{1 - \mu L}\lVert\theta_i - \theta_{i+j} \lVert^2 + \frac{2\mu}{L} \mathbb{E}[\lVert d_{i+j+1}\lVert^2 | \mathcal{F}_{i+j}]  + \frac{2\mu}{L} \lVert\nabla J(\theta_i)\lVert^2 + \mu^2 \sigma^2 \\
 &\leq \frac{(1 + \mu L)^2}{1 - \mu L}\lVert\theta_i - \theta_{i+j} \lVert^2 + \frac{2D^2\mu^3}{L}  + \frac{2\mu}{L} \lVert\nabla J(\theta_i)\lVert^2 + \mu^2 \sigma^2. \\
    \end{aligned}
\end{equation*}
Now we want to condition on $\theta_i \in \mathcal{H}$ to obtain
\begin{equation*}
\begin{aligned}
    \mathbb{E}[\lVert\theta_i - \theta_{i+j+1}\lVert^2 | \theta_i \in \mathcal{H}] 
 &\leq \mathbb{E}[\frac{(1 + \mu L)^2}{1 - \mu L}\lVert\theta_i - \theta_{i+j} \lVert^2 | \theta_i \in \mathcal{H}] + \frac{2D^2\mu^3}{L}  + \frac{2\mu^2 }{ L}\cdot \mu (L \sigma^2 + D^2 \mu)(1 + \frac{1}{\delta})  + \mu^2 \sigma^2 \\
 &\leq \frac{(1 + \mu L)^2}{1 - \mu L} \mathbb{E}[\lVert\theta_i - \theta_{i+j} \lVert^2 | \theta_i \in \mathcal{H}] + O(\mu^2).
\end{aligned}
\end{equation*}
 Then we can evaluate this recursive formula starting at $j = 0$, since $\mathbb{E}[\lVert\theta_i - \theta_i\lVert^2] = 0$, to arrive at
 \begin{equation*}
     \begin{aligned}
         \mathbb{E}[\lVert\theta_i - \theta_{i+j+1}\lVert^2 | \theta_i \in \mathcal{H}] 
 &\leq  O(\mu^2) \sum_{n = 0}^{j-1} \Big( \frac{(1 + \mu L)^2}{1 - \mu L} \Big)^n \\
 &\leq  O(\mu^2) \frac{1 - (\frac{(1 + \mu L)^2}{1 - \mu L})^j }{1 - \frac{(1 + \mu L)^2}{1 - \mu L}} \\
 &=  O(\mu^2) \frac{(1 - \mu L)(( \frac{(1 + \mu L)^2}{1 - \mu L})^j -1)}{1 + 2 \mu L + \mu^2 L^2 - 1 + \mu L} \\
 &=  O(\mu) \frac{(1 - \mu L)(( \frac{1 + 2 \mu L + \mu^2 L^2}{1 - \mu L})^j -1)}{3 L + \mu L^2 } \\
 &\leq O(\mu) \frac{( \frac{(1 + \mu L)^2}{1 - \mu L})^j}{3 L } \leq O(\mu) \frac{( \frac{(1 + \mu L)^2}{1 - \mu L})^{\frac{C}{\mu}}}{3 L} .
     \end{aligned}
 \end{equation*}
By Lemma \ref{lemma:minilemma}, this gives us
$$\mathbb{E}[\lVert\theta_i - \theta_{i+j+1}\lVert^2 | \theta_i \in \mathcal{H}]  \leq O(\mu).$$

Now we want to show the fourth moment bound (\ref{eq:mainbound2}), restated below
$$\mathbb{E}[\lVert\theta_i - \theta_{i +  j + 1}\lVert^4 | \theta_i \in \mathcal{H}] \leq O(\mu^2).$$
We use the inequality $\lVert a + b\lVert^4 \leq \lVert a \lVert^4 + 3 \lVert b\lVert^4 + 8 \lVert a \lVert^2 \lVert b\lVert^2 + 4 \lVert a\lVert^2 (a^T b)$ to expand the expression as follows
\begin{equation}
\label{eq:bound2}
    \begin{aligned}
        \lVert\theta_i - \theta_{i +  j + 1}\lVert^4 =& \lVert\theta_i - \theta_{i+j} - \mu \nabla J(\theta_{i + j}) - \mu \xi_{i+j+1} - \mu d_{i+j+1}\lVert^4 \\
        \leq& \lVert \theta_i - \theta_{i+j} - \mu \nabla J(\theta_{i + j}) - \mu d_{i+j+1}\lVert^4 + 3 \mu^4\lVert \xi_{i+j+1}\lVert^4 \\
        &+ 8 \mu^2 \lVert\theta_i - \theta_{i+j} - \mu \nabla J(\theta_{i + j}) - \mu d_{i+j+1}\lVert^2 \cdot  \lVert \xi_{i+j+1}\rVert^2 \\
        &+ 4 \lVert\theta_i - \theta_{i+j} - \mu \nabla J(\theta_{i + j}) - \mu d_{i+j+1}\lVert^2 (\theta_i - \theta_{i+j} - \mu \nabla J(\theta_{i + j}) - \mu d_{i+j+1})^T \xi_{i+j+1}.
    \end{aligned}
\end{equation}
We first consider the first term on the right hand side of (\ref{eq:bound2}) and decompose it via Jensen's inequality:
\begin{equation}
\label{eq:bound2-1}
    \begin{aligned}
        \lVert\theta_i - \theta_{i+j} - \mu \nabla J(\theta_{i + j}) - \mu d_{i+j+1}\lVert^4 \leq& \frac{1}{(1 - \mu L)^3} \lVert\theta_i - \theta_{i + j} - \mu \nabla J(\theta_{i+j}) + \mu \nabla J(\theta_i)\lVert^4\\
        &+  \frac{\mu}{ L^3} \lVert\nabla J(\theta_i) + d_{i+j+1} \lVert^4\\
        \leq& \frac{(1 + \mu L)^4}{(1 - \mu L)^3} \lVert\theta_i - \theta_{i + j} \lVert^4 +  \frac{8\mu}{ L^3} \lVert\nabla J(\theta_i)\lVert^4 + \frac{8\mu}{ L^3}\lVert d_{i+j+1} \lVert^4.\\
    \end{aligned}
\end{equation}
We then consider the third term on the right hand side of (\ref{eq:bound2}). From the analysis above, we have
\begin{equation}
\label{eq:bound2-2}
    \lVert\theta_i - \theta_{i+j} - \mu \nabla J(\theta_{i + j}) - \mu d_{i+j+1}\lVert^2 \leq \frac{(1 + \mu L)^2 }{1 - \mu L}\lVert\theta_i - \theta_{i+j}\lVert^2 + \frac{2\mu}{ L} \lVert\nabla J(\theta_i)\lVert^2 + \frac{2\mu}{ L}\lVert d_{i+j+1}\lVert^2 .
\end{equation}
Now we can plug (\ref{eq:bound2-1}) and (\ref{eq:bound2-2}) into (\ref{eq:bound2}) to obtain 
\begin{equation*}
    \begin{aligned}
        \lVert \theta_i - \theta_{i+j+1} \rVert^4 \leq& \frac{(1 + \mu L)^4}{(1 - \mu L)^3} \lVert\theta_i - \theta_{i + j} \lVert^4 +  \frac{8\mu}{ L^3} \lVert\nabla J(\theta_i)\lVert^4 + \frac{8\mu}{ L^3}\lVert d_{i+j+1} \lVert^4  + 3 \mu^4 \lVert \xi_{i+j+1} \rVert^4 \\
        &+ 8 \mu^2 \lVert \xi_{i + j+ 1} \rVert ^2 \left( \frac{(1 + \mu L)^2 }{1 - \mu L}\lVert\theta_i - \theta_{i+j}\lVert^2 + \frac{2\mu}{ L} \lVert\nabla J(\theta_i)\lVert^2 + \frac{2\mu}{ L}\lVert d_{i+j+1}\lVert^2 \right) \\
        &+ 4 \lVert\theta_i - \theta_{i+j} - \mu \nabla J(\theta_{i + j}) - \mu d_{i+j+1}\lVert^2 (\theta_i - \theta_{i+j} - \mu \nabla J(\theta_{i + j}) - \mu d_{i+j+1})^T \xi_{i+j+1}
    \end{aligned}
\end{equation*}
When we take the expectation on both sides, the cross term with $\xi_{i+j+1}$ disappears, and we have by (\ref{eq:noisebound2}), (\ref{eq:noisebound4}), (\ref{eq:momentbound4}) and (\ref{eq:momentboundboth})
\begin{equation*}
    \begin{aligned}
        \mathbb{E}[\lVert\theta_i - \theta_{i +  j + 1}\rVert^4 | \mathcal{F}_{i + j}] \leq& \frac{(1 + \mu L)^4}{(1 - \mu L)^3} \lVert\theta_i - \theta_{i + j} \lVert^4 +  \frac{8\mu}{ L^3} \lVert\nabla J(\theta_i)\lVert^4 + \frac{8\mu}{ L^3}\mathbb{E}[\lVert d_{i+j+1} \lVert^4 | \mathcal{F}_{i + j}]  + 12 \mu^4 \sigma^4  \\
        &+  \frac{8 \mu^2 \sigma^2 (1 + \mu L)^2 }{1 - \mu L}\lVert\theta_i - \theta_{i+j}\lVert^2 + \frac{16 \mu^3 \sigma^2}{ L} \lVert\nabla J(\theta_i)\lVert^2 \\
        &+ \frac{16 \mu^3}{ L} \mathbb{E}[\lVert \xi_{i + j + 1} \rVert^2 \cdot \lVert d_{i+j+1}\lVert^2 | \mathcal{F}_{i + j}]   \\
    \end{aligned}
\end{equation*}
\begin{equation*}
    \begin{aligned}
        \mathbb{E}[\lVert\theta_i - \theta_{i +  j + 1}\rVert^4 | \mathcal{F}_{i + j}] \leq& \frac{(1 + \mu L)^4}{(1 - \mu L)^3} \lVert\theta_i - \theta_{i + j} \lVert^4 +  \frac{8\mu}{ L^3} \lVert\nabla J(\theta_i)\lVert^4 + \frac{8 D^4 \mu^5 }{ L^3}  + 12 \mu^4 \sigma^4  \\
        &+  \frac{8 \mu^2 \sigma^2 (1 + \mu L)^2 }{1 - \mu L}\lVert\theta_i - \theta_{i+j}\lVert^2 + \frac{16 \mu^3 \sigma^2}{ L} \lVert\nabla J(\theta_i)\lVert^2 + \frac{16 \sigma^2 D^2 \mu^5}{ L}  \\
    \end{aligned}
\end{equation*}
Now we take expectation conditioned on $\theta_i \in \mathcal{H}$, allowing us to use the bound (\ref{eq:mainbound1}) derived before on $\lVert\theta_i - \theta_{i+j}\lVert^2$ to obtain
\begin{equation*}
    \begin{aligned}
        \mathbb{E}[\lVert\theta_i - \theta_{i +  j + 1}\lVert^4  | \theta_i \in \mathcal{H}] \leq& \frac{(1 + \mu L)^4}{(1 - \mu L)^3} \mathbb{E}[\lVert\theta_i - \theta_{i + j} \lVert^4 | \theta_i \in \mathcal{H}] +  \frac{8\mu}{ L^3} \mathbb{E}[\lVert\nabla J(\theta_i)\lVert^4 | \theta_i \in \mathcal{H}] \\
        &+ \frac{8\mu^2 \sigma^2 (1 + \mu L )^2}{1 - \mu L}\mathbb{E}[\lVert\theta_i - \theta_{i+j}\lVert^2 | \theta_i \in \mathcal{H}] + \frac{16 \mu^3 \sigma^2}{ L} \mathbb{E}[ \lVert\nabla J(\theta_i)\lVert^2 | \theta_i \in \mathcal{H}] + O(\mu^4) \\
        \leq& \frac{(1 + \mu L)^4}{(1 - \mu L)^3} \mathbb{E}[\lVert\theta_i - \theta_{i + j} \lVert^4 | \theta_i \in \mathcal{H}] + O(\mu^3)
    \end{aligned}
\end{equation*}
Then we can evaluate this recursive expression as follows
\begin{equation*}
    \begin{aligned}
        \mathbb{E}[\lVert\theta_i - \theta_{i +  j + 1}\lVert^4  | \theta_i \in \mathcal{H}] &\leq O(\mu^3) \sum_{n = 0}^{j-1} (\frac{(1 + \mu L)^4}{(1 - \mu L)^3})^n \\
        &= O(\mu^3) \frac{1 - (\frac{(1 + \mu L)^4}{(1 - \mu L)^3})^j}{1 - \frac{(1 + \mu L)^4}{(1 - \mu L)^3}} \\
        &= O(\mu^3) \frac{((\frac{(1 + \mu L)^4}{(1 - \mu L)^3})^j - 1)(1 - \mu L)^3}{(1 + \mu L)^4 - (1 - \mu L)^3} \leq O(\mu^2) \frac{(\frac{(1 + \mu L)^4}{(1 - \mu L)^3)})^j}{7 L + 3 \mu L^2 + 5 \mu^2 L^3 + \mu^3 L^4}
    \end{aligned}
\end{equation*}
By Lemma \ref{lemma:minilemma}, this gives us
$$\mathbb{E}[\lVert\theta_i - \theta_{i +  j + 1}\lVert^4  | \theta_i \in \mathcal{H}] \leq  O(\mu^2).$$

Finally, we want to bound (\ref{eq:mainbound3}), restated below
\begin{equation*}
    \mathbb{E}[\lVert\theta_{i+j} - \theta'_{i+j}\lVert^2 | \theta_i \in \mathcal{H}] \leq O(\mu).
\end{equation*}
First we expand the expression as follows using the definition of $\theta$ and $\theta'$
\begin{equation*}
    \begin{aligned}
        \lVert\theta_{i+j+1} &- \theta'_{i+j+1}\lVert^2 = \lVert\theta_{i+j} - \theta'_{i+j} + \mu \nabla J(\theta_{i+j}) + \mu d_{i+j+1}  - \mu \nabla J(\theta_i) - \mu \nabla^2 J(\theta_i) (\theta'_{i+j} - \theta_i) \lVert^2 \\
        &= \lVert(I + \mu \nabla^2 J(\theta_i))(\theta_{i+j} - \theta'_{i+j}) + \mu \nabla^2 J(\theta_i) (\theta_i - \theta_{i+j}) + \mu \nabla J(\theta_{i+j}) - \mu \nabla J(\theta_i) + \mu d_{i+j+1} \lVert^2
    \end{aligned}
\end{equation*}
Define $H_{i+j} = \int_0^1 \nabla^2 J ((1 - t) \theta_{i+j} + t \theta_i) dt$, then we can plug this into the expression and expand via Jensens's inequality to obtain
\begin{equation*}
    \begin{aligned}
        \lVert\theta_{i+j+1} &- \theta'_{i+j+1}\lVert^2 = \lVert (I + \mu \nabla^2 J(\theta_i)) (\theta_{i+j} - \theta'_{i+j}) + \mu (\nabla^2 J(\theta_i) - H_{i+j} ) (\theta_i - \theta_{i+j} ) + \mu d_{i+j+1} \lVert ^2 \\
        &\leq \frac{1}{1 - \mu L} \lVert (I + \mu \nabla^2 J(\theta_i)) (\theta_{i+j} - \theta'_{i+j})\lVert ^2 + \frac{\mu}{L} \lVert  (\nabla^2 J(\theta_i) - H_{i+j} ) (\theta_i - \theta_{i+j} ) +  d_{i+j+1} \lVert ^2 \\
        &\leq \frac{1}{1 - \mu L} \lVert (I + \mu \nabla^2 J(\theta_i)) (\theta_{i+j} - \theta'_{i+j})\lVert ^2 + \frac{2\mu}{L} \lVert  (\nabla^2 J(\theta_i) - H_{i+j} ) (\theta_i - \theta_{i+j} )\lVert ^2 +  \frac{2 \mu}{L} \lVert d_{i+j+1} \lVert ^2
    \end{aligned}
\end{equation*}
As observed in  \cite{vlaski2019secondorder}, we have
\begin{equation}
\label{eq:bound3yup}
    \begin{aligned}
        \lVert \nabla^2 J(\theta_i) - H_{i+j}\lVert  &= \lVert \nabla^2 J(\theta_i) - \int_{0}^1 \nabla^2 J((1 - t) \theta_{i+j} + t \theta_i) dt\lVert \\
        &= \lVert \int_{0}^1 \nabla^2 J(\theta_i) -  \nabla^2 J((1 - t) \theta_{i+j} + t \theta_i) dt \lVert \\
        &\leq \int_{0}^1 \lVert \nabla^2 J(\theta_i) -  \nabla^2 J((1 - t) \theta_{i+j} + t \theta_i) \lVert  dt \\
        &\leq \chi \int_0^1 \lVert (1 - t) \theta_{i} - (1 - t)\theta_{i + j} \lVert  dt \leq \frac{\chi}{2} \lVert  \theta_i - \theta_{i + j}\lVert ,
    \end{aligned}
\end{equation}
which implies
\begin{equation}
\label{eq:bound3aux}
    \lVert  (\nabla^2 J(\theta_i) - H_{i+j} ) (\theta_i - \theta_{i+j} )\lVert ^2 \leq \frac{\chi}{2} \lVert \theta_i - \theta_{i+j}\lVert ^4.
\end{equation}
We can plug (\ref{eq:bound3aux}) back into (\ref{eq:bound3yup}) and take expectation of both sides, conditioned on $\theta_i \in \mathcal{H}$. When we apply the fourth moment bound from Lemma \ref{lemma:bounds}, we obtain
$$\mathbb{E}[\lVert \theta_{i+j+1} - \theta'_{i+j+1}\lVert ^2 | \theta_i \in \mathcal{H}] \leq \frac{(1 + \mu L)^2 }{(1 - \mu L)} \mathbb{E}[\lVert \theta_{i+j} - \theta'_{i+j}\lVert ^2 | \theta_i \in \mathcal{H}] + O(\mu^3).$$
This is the same recursion as in the proof of (\ref{eq:mainbound1}), so again from Lemma \ref{lemma:minilemma} we have
$$\mathbb{E}[\lVert \theta_{i+j+1} - \theta'_{i+j+1}\lVert ^2 |\theta_i \in \mathcal{H}] \leq O(\mu^2).$$
\end{proof}

\section{Noise and Bias Bounds for Vanilla Policy Gradient}
\label{app:VPG}

To apply Theorem \ref{theorem:main}, we first show in Lemma \ref{lemma:boundednoise} that the gradient estimator and the second and fourth moment of the noise are bounded. Then we show that the  deterministic bias term $d_{t+1}$ is bounded via Lemma \ref{lemma:biasmu}. This allows us to directly conclude the results of Theorem \ref{theorem:1epsilon}.

\begin{lemma} 
\label{lemma:boundednoise}
The gradient noise process $\{\xi_t\}_{t \geq 0}$ satisfies 
$$\mathbb{E}[\xi_{t+1} | \mathcal{F}_{t} ] = \mathbb{E}[\hat{G}^{VPG}(\theta_{t} ; \tau_t) - \nabla J_H (\theta_t) | \mathcal{F}_{t} ] = 0.$$
In addition, let $\sigma = \frac{G \mathcal{R}_{max}}{(1 - \gamma)^2}$. Then we have the following bounds
$$\lVert\hat{G}^{VPG}(\theta_{t} ; \tau_t)\rVert \leq \sigma ,$$
$$\mathbb{E}[\lVert \xi_{t+1} \lVert^2 | \mathcal{F}_{t}] \leq \sigma^2  ,$$
$$\mathbb{E}[\lVert \xi_{t+1} \lVert^4 | \mathcal{F}_{t}] \leq 4\sigma^4 . $$
    
\end{lemma}

\begin{proof}
\begin{equation*} \begin{aligned}
\lVert \hat{G}^{VPG}(\theta ; \tau) \rVert &= \lVert \sum_{h = 0}^{H-1} \nabla_\theta \log \pi_\theta (a_h | s_h) \sum_{t = h}^{H-1} \gamma^t \mathcal{R} ( s_t, a_t) \lVert\\
&\leq \sum_{h = 0}^{H-1} \lVert \nabla_\theta \log \pi_\theta (a_h | s_h) \sum_{t = h}^{H-1} \gamma^t \mathcal{R} ( s_t, a_t) \lVert\\
 &\leq  \sum_{h = 0}^{H-1} \lVert\nabla_\theta \log \pi_\theta (a_h | s_h)  \rVert \gamma^h \sum_{t = h}^{H-1}  \gamma^{t- h} \mathcal{R}_{max} \\
 &\leq  \sum_{h = 0}^{H-1} \lVert \nabla_\theta \log \pi_\theta (a_h | s_h) \rVert \gamma^h \frac{\mathcal{R}_{max}}{1 - \gamma }\\
 &\leq \frac{G \mathcal{R}_{max}}{(1 - \gamma)^2}
\end{aligned} \end{equation*}
Then the rest of the bounds follow from Lemma \ref{lemma:sigma}. Thanks to reviewer feedback, we note that the bound on the noise variance $\mathbb{E}[ \lVert \xi_{t+1} \rVert^2]$ can be tightened by a factor of $\frac{1}{1 - \gamma}$ as shown in Lemma 4.2 of \cite{yuan21vanilla}.
\end{proof}

Before we can prove Lemma \ref{lemma:biasmu}, we require the following lemma from \cite{yuan21vanilla}.

\begin{lemma} (Lemma 4.5 from \cite{yuan21vanilla})
\label{lemma:boundedbias}
    For $D = \frac{G \mathcal{R}_{max}}{1 - \gamma}$, we have that the bias term $d_{i+1}$ is bounded such that 
    \begin{equation*}
        \lVert d_{i+1}\lVert = \lVert\nabla J(\theta_i) - \nabla J_H (\theta_i)\lVert \leq D (\frac{1}{1 - \gamma} + H)^{1/2} \gamma^H.
    \end{equation*}
\end{lemma}

Now we can proceed with the proof of Lemma \ref{lemma:biasmu}.

\begin{lemma}
\label{lemma:biasmu}
 For $H =\frac{1}{\log \frac{1}{\gamma}} \cdot O (\log(\frac{1}{\mu}))$ where $\mu \to 0$, we have that the gradient bias is deterministically bounded as follows
\begin{equation*}
   \lVert d_{t+1} \rVert = \lVert\nabla J(\theta_t) - \nabla J_H (\theta_t) \rVert \leq D (\frac{1}{1-\gamma} + H)^{1/2} \gamma^H \leq D\mu 
\end{equation*}
where $D = \frac{G \mathcal{R}_{max}}{1 - \gamma}$.
\end{lemma}

\begin{proof}
We have the bound on the bias in terms of $H$ from Lemma  \ref{lemma:boundedbias}. We want to choose $H$ large enough so that 
    $$D(\frac{1}{1-\gamma } + H)^{1/2} \gamma^H \leq D \mu$$
    $$(\frac{1}{1-\gamma } + H)^{1/2} \gamma^H \leq \mu.$$
We begin by finding the approximate solution to the following equation using asymptotic expansion
    $$(\frac{1}{1-\gamma } + H)^{1/2} \gamma^H = \mu$$
    $$\frac{1}{2}\log (\frac{1}{1-\gamma } + H) + H \log \gamma = \log \mu$$
    $$\frac{1}{2}\log (\frac{1}{1-\gamma } + H) - H \log \frac{1}{\gamma} = -\log \frac{1}{\mu}$$
    $$ H \log \frac{1}{\gamma} - \frac{1}{2}\log (\frac{1}{1-\gamma } + H) = \log \frac{1}{\mu}.$$
    Now we use the method of dominant balance, treating $\mu$ as a small parameter. We have that $\log \frac{1}{\mu}$ and $H \log \frac{1}{\gamma}$ must balance each other out, so 
    $$H \sim \frac{\log \frac{1}{\mu}}{\log \frac{1}{\gamma}} = \frac{\log \mu}{\log \gamma}.$$
    We consider the next term in our asymptotic expansion of $H$, assuming that it is much smaller than the first term
    $$H \sim \frac{\log \mu}{\log \gamma} + x_1.$$
    We substitute this into the inequality to obtain
    $$(\frac{\log \mu}{\log \gamma} + x_1) \log \frac{1}{\gamma} - \frac{1}{2}\log(\frac{1}{1 - \gamma} + \frac{\log \mu}{\log \gamma} + x_1) = \log\frac{1}{\mu}$$
    $$x_1 \log \frac{1}{\gamma} - \frac{1}{2} \log(\frac{1}{1 - \gamma} + \frac{\log \mu}{\log \gamma} + x_1) = 0$$
    $$x_1 = \frac{\log(\frac{1}{1 - \gamma} + \frac{\log \mu}{\log \gamma} + x_1)}{2\log \frac{1}{\gamma}}.$$
    Since we assume $x_1$ is much smaller than $\frac{\log \mu}{\log \gamma}$ we have 
    $$x_1 \sim \frac{\log(\frac{1}{1 - \gamma} + \frac{\log \mu}{\log \gamma})}{2\log \frac{1}{\gamma}}.$$
    Our asymptotic expansion is $H = \frac{\log \mu}{\log \gamma} + O(\log (\frac{\log \mu}{\log \gamma}))$. So we can pick $H = O(\frac{\log \mu}{ \log \gamma})$ to achieve our inequality.    
\end{proof}

\section{Proof of Theorem \ref{theorem:bhandarialt}}
\label{appendix:td(0)}
\subsection{Key Lemmas and Proof Sketch}
In Theorem \ref{theorem:bhandarialt}, we establish the convergence of $Q_{\bar{w}_K}$ to $Q_{w^*}$ under constant time steps.
Our approach will mirror that of \cite{bhandariTD} by establishing a recurrence relation for the iterates and then bounding the bias induced by Markovian sampling. The key challenge is characterizing the distance between the initial state distribution and the stationary distribution in terms of the mixing rate.

We define $\bar{g}(w)$ as the expectation of of the semigradient $g_t(w)$ with respect to the stationary distribution of the Markov chain as follows
\begin{equation*}
    \begin{aligned}
        \bar{g}_(w) &= \mathbb{E}_{\eta_{\pi}} [g_t(w)] = \sum_{s, s', a, a'} \eta_{\pi}(s, a) \mathcal{P}(s, a, s', a') (\mathcal{R}(s, a) + \gamma \phi(s', a')^\intercal w - \phi (s, a)^\intercal w) \phi(s, a).
    \end{aligned}
\end{equation*}
We also define $\zeta_t$ to represent the bias from Markovian sampling as follows
$$\zeta_t(w) = (g_t(w) - \bar{g}(w))^\intercal(w - w^*).$$

First, The following lemma, which is Lemma 6 and 10 from \cite{bhandariTD}, uniformly bounds the norm of the semi-gradient and Markov bias term $\zeta_t$.

\begin{lemma}
\label{lemma:6and10}
Let $F = \mathcal{R}_{max} + 2R$. Then $R \leq \frac{F}{2}$ and for all $t \geq 0$ we have
$$\lVert g_t(w)\lVert _2 \leq \mathcal{R}_{\max} + 2 \lVert w\lVert _2 \leq F$$
In addition, for all $w \in \Theta$, the gradient bias is bounded such that 
$$|\zeta_t(w)| \leq 2 F^2$$
$$|\zeta_t(w) - \zeta_t(w')| \leq 6 F \lVert w - w'\lVert _2.$$

\end{lemma}

Then we obtain the following lemma for general nonstationary Markov chains, which differs from Lemma 9 in \cite{bhandariTD} by a factor of 2.

\begin{lemma}
\label{lemma:lemma9bhandari}
Consider two random variables $X$ and $Y$ such that 
$$X \to s_t \to s_{t + \tau} \to Y$$
forms a Markov chain for some fixed $t \geq 0$ and $\tau > 0$. Assume the Markov chain mixes at a uniform geometric rate as in Assumption \ref{ass:ergodic}. Let $X'$ and $Y'$ denote independent copies drawn from the marginal distributions of $X$ and $Y$, so $\mathbb{P})X' = \cdot, Y' = \cdot) = \mathbb{P}(X = \cdot) \otimes \mathbb{P}(Y = \cdot)$. Then, for any bounded function $v$,
$$| \mathbb{E}[v(X, Y)] - \mathbb{E}[v(X', Y')]| \leq 4 \lVert v\lVert _{\infty} (m r^{\tau}),$$
\end{lemma}
where $\lVert v \rVert_\infty = \sup_{x} | f(x)|$.

\textit{Proof.} See Appendix \ref{sec:lemma9bhandari}. 

Now in the following key lemma, we  apply Lemma \ref{lemma:lemma9bhandari} to bound $\zeta_t(w_t)$ with respect to exponential mixing. Although we follow the proof of Lemma 11 in \cite{bhandariTD}, it is not sufficient to directly carry the factor of $2$ over from Lemma \ref{lemma:lemma9bhandari} because we  need to account for the fact that the marginal distribution of each observation $O_t$ is now time-dependent and not equal to the stationary distribution.

\begin{lemma}
\label{lemma:modified11}
Consider a non-increasing step-size sequence $\alpha_0 \geq \alpha_1 \geq ... \geq \alpha_T$. Let $\tau_0 = \tau^{\text{\tiny mix}}(\alpha_T)$. Fix any $t \leq T$ and set $t^* = \max \{ 0, t - \tau_0 \}$. Then,
$$\mathbb{E}[\zeta_t (w_t)] \leq F^2(8 + 6 \tau_0) \alpha_{t^*} + 10 F^2 mr^t.$$
\end{lemma}

\textit{Proof.} See Appendix \ref{sec:modified11}.

Finally, we have the following lemma from \cite{bhandariTD} that establishes a recursion for the distance between the iterates and the limit point $w^*$.
\begin{lemma}
\label{lemma:bhand_8}
    With probability $1$, for every $t \in \mathbb{N}_0$,
    $$\lVert w^* - w_{t+1} \rVert^2 \leq \lVert w^* - w_t \rVert^2 - 2 \alpha_t (1 - \gamma) \lVert Q_{w^*} - Q_{w_t} \rVert_{\eta_{\pi}}^2 + 2 \alpha_t \zeta_t(w_t) + \alpha_t^2 F^2. $$
\end{lemma}
\textit{Proof. }See the proof of Lemma 8 in \cite{bhandariTD}.

These key lemmas allow us to proceed with the proof of Theorem \ref{theorem:bhandarialt}.

\subsection{Proof of Theorem \ref{theorem:bhandarialt}}
\begin{proof}
Rearranging the terms of the inequality in  Lemma \ref{lemma:bhand_8} and summing from $t = 0$ to $K-1$, we arrive at 
$$2 \alpha_0 (1 - \gamma) \sum_{t = 0}^{K - 1} \mathbb{E}[\lVert Q_{w^*} - Q_{w_t}\lVert ^2_{\eta_{\pi}} ] \leq \lVert w^* - w_0\lVert ^2_2 + F^2  + 2 \alpha_0 \sum_{t = 0}^{K - 1} \mathbb{E}[\zeta_t(w_t)].$$
Then we can use the bound on $\zeta_t(w_t)$ from Lemma \ref{lemma:modified11} and the fact that our step-sizes are constant to obtain 
\begin{equation*}
    \begin{aligned}
        2 \alpha_0 (1 - \gamma) \sum_{t = 0}^{K - 1} \mathbb{E}[\lVert Q_{w^*} - Q_{w_t}\lVert ^2_{\eta_{\pi}} ] &\leq \lVert w^* - w_0\lVert ^2_2 + F^2  + 2 \alpha_0 \sum_{t = 0}^{K - 1} F^2 (8 + 6 \tau_0) \alpha_0 + 2\alpha_0 \sum_{t = 0}^{K - 1} 10 F^2 mr^t \\
        &\leq \lVert w^* - w_0\lVert ^2_2 + F^2  + 2 \alpha_0^2 K F^2 (8 + 6 \tau_0) + \frac{20 F^2 m \alpha_0}{1 - r} .\\
        \end{aligned}
    \end{equation*}
Now we can divide both sides by $2 \alpha_0 (1 - \gamma)$ and substitute $\alpha_0 = \frac{1}{\sqrt{K}}$ to obtain
    \begin{equation*}
    \begin{aligned}
        \sum_{t = 0}^{K - 1} \mathbb{E}[\lVert Q_{w^*} - Q_{w_t}\lVert ^2_{\eta_{\pi}} ] &\leq \frac{\lVert w^* - w_0\lVert ^2_2 + F^2}{2 \alpha_0 (1 - \gamma)}  + \frac{\alpha_0 K F^2 (8 + 6 \tau_0)}{1 - \gamma} + \frac{10 F^2 m}{(1 - r)(1 - \gamma)} \\
        &= \frac{\sqrt{K}(\lVert w^* - w_0\lVert ^2_2 + F^2)}{2 (1 - \gamma)}  + \frac{\sqrt{K} F^2 (8 + 6 \tau_0)}{1 - \gamma} + \frac{10 F^2 m}{(1 - r)(1 - \gamma)}\\
        &= \frac{\sqrt{K}(\lVert w^* - w_0\lVert ^2_2 + 17 F^2 + 12 F^2 \tau_0)}{2 (1 - \gamma)}  + \frac{10 F^2 m}{(1 - r)(1 - \gamma)}.
    \end{aligned}
\end{equation*}
Finally, we divide both sides by $K$ and use Jensen's inequality to obtain our final result
\begin{equation*}
    \begin{aligned}
        \mathbb{E}[\lVert Q_{w^*} - Q_{\bar{w}_K}\lVert ^2_{\eta_{\pi}}]  \leq \frac{1}{K} \sum_{t = 0}^{K - 1} \mathbb{E}[\lVert Q_{w^*} - Q_{w_t}\lVert ^2_{\eta_{\pi}} ] \leq \frac{\lVert w^* - w_0\lVert ^2_2 + F^2 (17 + 12\tau_0)}{2 (1 - \gamma) \sqrt{K}}  + \frac{10 F^2 m}{(1 - r)(1 - \gamma) K}. \\
    \end{aligned} 
\end{equation*}
\end{proof}

\subsection{Proof of Lemma \ref{lemma:lemma9bhandari}}
\label{sec:lemma9bhandari}

We first require the following auxiliary lemma.
\begin{lemma}
\label{lemma:markovstationary}
For any Markov chain $\{s_t\}$ with stationary distribution $\eta$ and finite state space $\mathcal{S}$,
$$d_{TV} (\eta, \mathbb{P}(s_{t + \tau} = \cdot )) \leq 
\sup_{s \in S} d_{TV} (\eta, \mathbb{P}(s_{t + \tau} = \cdot | s_0 = s)). $$
\end{lemma}

\begin{proof}
It follows from proof by induction that for a general convex function $f$, if $f(x_n) \geq f(x_i)$ for all $x_i \in \{x_1, ... x_n\}$ then $f(x_n) \geq f(x)$ for $x = \sum_{i = 1}^n \alpha_i x_n$ where $\sum_{i = 1}^n \alpha_i = 1$. In other words, $f(x_n)$ is greater than $f(x)$ for any $x$ that is a convex combination of the other $\{x_1,... x_n\}$. 

Now let $f(x) = \frac{1}{2} \sum_{s_i \in \mathcal{S}} |x^\intercal P^{t + \tau}(s_i) - \eta(s_i) |$, which is a convex function, and let $e_i \in \mathbb{R}^{S}$ represent the unit vector that is $1$ at index $i$ and $0$ everywhere else. Then we have
$$d_{TV} (\eta, \mathbb{P}(s_{t + \tau} = \cdot )) = f(\rho_0)$$
$$d_{TV} (\eta, \mathbb{P}(s_{t + \tau} = \cdot |s_0 = s_i)) = f(e_i).$$
Since any initial state distribution $\rho_0$ will be a convex combination of the $e_i$ unit vectors, we have shown the result.
\end{proof}

Now we can proceed with the proof of Lemma \ref{lemma:lemma9bhandari}.

\begin{proof}
Let $h = \frac{v}{2 \lVert v\lVert _{\infty}}$ denote the function $v$ rescaled to take values in $[-1/2, 1/2]$. Then we can follow the steps of Lemma 9 in \cite{bhandariTD} to arrive at 
\begin{equation}
\label{eq:lemma9}
|\mathbb{E}[h(X, Y)] - \mathbb{E}[h(X', Y')]| \leq \sum_{s \in \mathcal{S}} \mathbb{P}(s_t = s) d_{TV}(\mathbb{P}(s_{t + \tau} = \cdot | s_t = s), \mathbb{P}(s_{t + \tau} = \cdot ).
\end{equation}
We can bound $d_{TV} (\mathbb{P}(s_{t + \tau} = \cdot | s_t = s), \mathbb{P}(s_{t + \tau} = \cdot ))$ as follows
\begin{equation}
\label{eq:lemma9_2}
    d_{TV} (\mathbb{P}(s_{t + \tau} = \cdot | s_t = s), \mathbb{P}(s_{t + \tau} = \cdot )) \leq d_{TV} (\mathbb{P}(s_{t + \tau} = \cdot | s_t = s), \eta_\pi) + d_{TV} (\eta_\pi, \mathbb{P}(s_{t + \tau} = \cdot ))
\end{equation}
because total variation distance is a norm and obeys the triangle inequality.
The first term on the right hand side of (\ref{eq:lemma9_2}) is bounded by exponential mixing
$$d_{TV} (\mathbb{P}(s_{t + \tau} = \cdot | s_t = s), \eta_{\pi}) \leq m r^\tau.$$
The second term can be bound as follows by Lemma \ref{lemma:markovstationary}
$$d_{TV} (\eta_{\pi}, \mathbb{P}(s_{t + \tau} = \cdot )) \leq 
\sup_{s \in S} d_{TV} (\eta_\pi, \mathbb{P}(s_{t + \tau} = \cdot | s_0 = s)) \leq m r^{t + \tau} .$$
Returning to (\ref{eq:lemma9_2}), we have
$$d_{TV} (\mathbb{P}(s_{t + \tau} = \cdot | s_t = s), \mathbb{P}(s_{t + \tau} = \cdot )) \leq m r^\tau + m r^{t + \tau} \leq 2 m r^{\tau},$$
and applying these inequalities to (\ref{eq:lemma9}), we have
$$|\mathbb{E}[v(X, Y)] - \mathbb{E}[v(X', Y')]| \leq 4 \lVert v\lVert _{\infty} mr^{\tau}.$$
\end{proof}

\subsection{Proof of Lemma \ref{lemma:modified11}}
\label{sec:modified11}

First we require the following auxiliary lemmas.

\begin{lemma}
\label{lemma:pq}
Let $P$ and $Q$ represent two different probability distributions. For some bounded function $f$, and random variable $X$, we have
$$|\mathbb{E}_P[f(X)] - \mathbb{E}_Q[f(X)]| \leq 2 \lVert f\lVert _{\infty} d_{TV}(P, Q)$$
\end{lemma}
\begin{proof}
This can be shown with the definition
$$d_{TV} (P, Q) = \sup_{v:\lVert v\lVert _\infty \leq \frac{1}{2}} |\int v dP - \int v dQ|$$
\end{proof}

\begin{lemma}
\label{lemma:stationary}
Let $O'' = (s, a, s', a')$ represent the observation tuple of consecutive state-action pairs where $(s, a)$ is drawn from the stationary distribution of the Markov chain $\eta_{\pi}$, and let $O_t = (s_t, a_t, s_{t+1}, a_{t+1})$ represent the observation tuple drawn at time $t$ from the Markov chain. Then
$$ d_{TV}(\mathbb{P}(O'' = \cdot), \mathbb{P}(O_t = \cdot)) = d_{TV}(\eta_\pi, \mathbb{P}(s_t = \cdot, a_t = \cdot)) \leq m r^t$$
    
\end{lemma}

Now we can proceed with the proof of Lemma \ref{lemma:modified11}, which mirrors the  framework of  Lemma 11 in \cite{bhandariTD}. The main difference in our proofs is in Step 2.

\begin{proof}
\textit{Step 1: Relate $\zeta_t (w_t)$ and $\zeta_t (w_{t- \tau})$}

We apply Lemma \ref{lemma:6and10} to bound $\zeta_t$ as follows
$$|\zeta_t (w_t) - \zeta_t(w_{t - \tau}) | \leq 6F \lVert w - w_{t - \tau}\lVert  \leq 6F^2 \sum_{i = t - \tau}^{t - 1} \alpha_i$$
\begin{equation}
\label{eq:lemma11_1}
    \zeta_t(w_t) \leq \zeta_t(w_{t - \tau}) + 6F^2 \sum_{i = t - \tau}^{t - 1} \alpha_i
\end{equation}
\textit{Step 2: Bound $\mathbb{E}[\zeta_t (w_{t - \tau})]$ using Lemma \ref{lemma:lemma9bhandari} and exponential mixing}

We denote by $O_t = (s_t, a_t, s_{t+1}, a_{t + 1})$ the observation tuple at each time step, and we overload the notation of $g_t$ and $\zeta_t$ to make clear their dependence on $O_t$ as follows
$$g_t(w, O_t) = (r(s_t, a_t) + \gamma \phi(s_{t+ 1}, a_{t+1})^\intercal w - \phi (s_{t}, a_t)^\intercal w) \phi(s_t, a_t)$$
$$\zeta(w, O_t) = (g_t(w, O_t) - \bar{g}(w))^{\intercal} (w - w^*)$$
To apply Lemma \ref{lemma:lemma9bhandari}, we consider random variables $w'_{t - \tau}$ and $O'_t$ drawn independently from their marginal distributions $\mathbb{P}(w_{t - \tau} = \cdot)$ and $\mathbb{P}(O_t = \cdot)$ respectively, such that their joint distribution is defined as follows 
$$\mathbb{P}(w'_{t + \tau} = \cdot, O'_t = \cdot) = \mathbb{P}(w_{t - \tau} = \cdot) \otimes \mathbb{P} (O_t = \cdot ) $$
Note that typically the random variables $w_{t - \tau}$ and $O_t$ are not independent. Now we want to bound $\mathbb{E}[\zeta(w'_{t-\tau}, O'_t)]$. We can use the law of total expectation as follows
\begin{equation*}
    \mathbb{E}[\zeta (w'_{t-\tau}, O'_t)] = \mathbb{E}[\mathbb{E}[\zeta (w'_{t-\tau}, O'_t) | w'_{t - \tau}]] 
\end{equation*}
Now we consider the conditional expectation term $\mathbb{E}[\zeta (w'_{t-\tau}, O'_t) | w'_{t - \tau}]$. Since $w'_{t - \tau}$ and $O'_t$ are independent, we have
\begin{equation*}
    \begin{aligned}
\mathbb{E}[\zeta (w'_{t-\tau}, O'_t) | w'_{t - \tau}] &= \mathbb{E}_{O'_t \sim \mathbb{P}(O_t = \cdot)}[\zeta (w'_{t-\tau}, O'_t) | w'_{t - \tau}] \\
&= \mathbb{E}_{O'_t \sim \mathbb{P}(O_t = \cdot)}[ (g_t(w'_{t - \tau}, O'_t) - \bar{g}(w'_{t - \tau}))^\intercal(w'_{t - \tau} - w^*)| w'_{t - \tau}] \\
&= ( \mathbb{E}_{O'_t \sim \mathbb{P}(O_t = \cdot)}[ g_t(w'_{t - \tau}, O'_t)] - \bar{g}(w'_{t - \tau}))^\intercal(w'_{t - \tau} - w^*) \\
&= ( \mathbb{E}_{O'_t \sim \mathbb{P}(O_t = \cdot)}[ g_t(w'_{t - \tau}, O'_t)] - \mathbb{E}_{O'' \sim \eta_{\pi}} [g(w'_{t - \tau}, O'')])^\intercal(w'_{t - \tau} - w^*) \\
&\leq \lVert   \mathbb{E}_{O'_t \sim \mathbb{P}(O_t = \cdot)}[ g_t(w'_{t - \tau}, O'_t)] - \mathbb{E}_{O'' \sim \eta_{\pi}} [g(w'_{t - \tau}, O'')] \lVert  \cdot \lVert  w'_{t - \tau} - w^* \lVert  \\
&\leq \lVert   \mathbb{E}_{O'_t \sim \mathbb{P}(O_t = \cdot)}[ g_t(w'_{t - \tau}, O'_t)] - \mathbb{E}_{O'' \sim \eta_\pi} [g(w'_{t - \tau}, O'')] \lVert  \cdot 2R, \\
    \end{aligned}
\end{equation*}
where the last inequality is due to the fact that $\lVert w\lVert  \leq R$. Here $O'_t$ is drawn from the time-dependent marginal distribution $\mathbb{P}(O_t = \cdot)$ whereas $O''$ is drawn from the stationary distribution of the Markov chain. Then by Lemma \ref{lemma:pq} and the fact that $\lVert g_t\lVert  \leq F$ and $R \leq \frac{F}{2}$, we can bound this difference in expectation by the total variation distance between the two distributions as follows 
\begin{equation*}
    \begin{aligned}
\mathbb{E}[\zeta (w'_{t-\tau}, O'_t) | w'_{t - \tau}] \leq 2 F^2 d_{TV}(\mathbb{P}(O'' = \cdot), \mathbb{P}(O_t = \cdot)). \\
    \end{aligned}
\end{equation*}
Then by Lemma \ref{lemma:markovstationary} and Lemma \ref{lemma:stationary} we have
\begin{equation*}
    \begin{aligned}
\mathbb{E}[\zeta (w'_{t-\tau}, O'_t) | w'_{t - \tau}] \leq 2 F^2 mr^t .\\
    \end{aligned}
\end{equation*}
Now we can apply Lemma \ref{lemma:lemma9bhandari} to bound $\mathbb{E}[\zeta_t(w_{t - \tau}), O_t)]$. By Lemma \ref{lemma:6and10}, we have $|\zeta(w, O_t)| \leq 2 F^2$. Since $\theta_{t - \tau} \to s_{t - \tau} \to s_t \to O_t$ forms a Markov chain, applying Lemma \ref{lemma:lemma9bhandari} yields 
$$|\mathbb{E}[\zeta(w_{t - \tau}, O_t)] - \mathbb{E}[\zeta(w'_{t - \tau}, O'_t) ]| \leq 8 F^2 m r^{\tau} $$
\begin{equation}
\label{eq:lemma11_2}
    |\mathbb{E}[\zeta(w_{t - \tau}, O_t)] | \leq 8 F^2 m r^{\tau} + 2 F^2 mr^t.
\end{equation}


\textit{Step 3: Combine terms.}

We combine (\ref{eq:lemma11_1}) and (\ref{eq:lemma11_2}) to arrive at 
$$\mathbb{E}[\zeta_t (w_t)] \leq 8 F^2 m r^{\tau} + 2 F^2 mr^t + 6 F^2 \tau \alpha_{t - \tau}.$$
Let $\tau_0 = \tau^{\text{\tiny mix}}(\alpha_T)$. For $t \leq \tau_0$, pick $\tau = t$, to arrive at the bound
$$\mathbb{E}[\zeta_t (w_t)] \leq 8 F^2 m r^{t} + 2 F^2 mr^t + 6 F^2 t \alpha_{0} \leq 10 F^2 m r^{t}  + 6 F^2 \tau_0 \alpha_{0}.$$
Now for $ T \geq t > \tau_0$ we pick $\tau = \tau_0$ to get the bound
$$\mathbb{E}[\zeta_t (w_t)] \leq 8 F^2 \alpha_T + 2 F^2 mr^t + 6 F^2 \tau_0 \alpha_{t - \tau_0} \leq F^2(8 + 6 \tau_0) \alpha_{t - \tau_0} + 2 F^2 mr^t.$$
    
\end{proof}

\section{Noise and Bias Bounds for Actor-Critic Policy Gradient}
\label{sec:appendACnoisebias}

Our general approach for actor-critic policy gradient is similar to that of vanilla policy gradient. Once again, we bound the noise and bias --- bounding $\xi_t$ in Lemma \ref{lemma:ACboundedbiasnoise}, $p_t$ in Lemma \ref{lemma:truncboundAC}, and $d_t$ in Lemma \ref{sec:FAbiasboundAC}.

We note that for vanilla policy gradient, the noise term $\xi_{t+1}$ is random due to the sampled trajectory $\tau_t$ whereas the bias term $d_{t+1}$ is deterministic conditioned on $\mathcal{F}_t$. In contrast, for the actor-critic algorithm, $\xi_{t+1}$ depends on two random variables, $\tau_t$ and the averaged critic parameter $\bar{w}_{K,t}$, whereas the bias term depends only on $\bar{w}_{K,t}$. Moreover, $\tau_t$ and $\bar{w}_{K,t}$ are generated from independently sampled trajectories and are therefore independent. We quantify these observations in the following lemma.

\begin{lemma}
\label{lemma:ACboundedbiasnoise}
The gradient noise process $\{\xi_t\}_{t\geq 0}$ satisfies
\begin{equation*}
    \mathbb{E}[\xi_{t+1} | \mathcal{F}_t] = 0
\end{equation*}
In addition, let $\sigma = \frac{GR}{1 - \gamma}$. Then we have the following bounds
\begin{equation*}
    \lVert  \hat{G}(\theta_t; \tau_t) \lVert  \leq \sigma,
\end{equation*}
\begin{equation*}
    \mathbb{E} [\lVert  \xi_{t+1} \lVert ^2 | \mathcal{F}_t] \leq \sigma^2
\end{equation*}
\begin{equation*}
    \mathbb{E} [\lVert  \xi_{t+1} \lVert ^4  | \mathcal{F}_t] \leq 4\sigma^4
\end{equation*}
\begin{equation*}
    \mathbb{E}[\lVert d_{t+1}\rVert^2 \lVert \xi_{t+1} \rVert^2 | \mathcal{F}_{t} ] \leq \sigma^2 \mathbb{E}[\lVert d_{t+1}\rVert^2 | \mathcal{F}_t]
\end{equation*}
\end{lemma}

\begin{proof}
As discussed earlier, $\xi_{t+1}$ depends on two random independent variables: $\tau_t$ and $\bar{w}_{K, t}$. We can overload notation and denote the gradient estimator as $\hat{G}(\theta_t ; \tau_t ; \bar{w}_{K, t})$. We therefore obtain
\begin{equation*}
    \begin{aligned}
        \mathbb{E}[\xi_{t+1} | \mathcal{F}_t] &= \mathbb{E}[\hat{G}(\theta_t ; \tau_t ; \bar{w}_{K, t}) - \mathbb{E}_{\tau} [\hat{G}(\theta_t ; \tau_t ; \bar{w}_{K, t})] | \mathcal{F}_t]\\
        &= \mathbb{E}[ \mathbb{E}[\hat{G}(\theta_t ; \tau_t ; \bar{w}_{K, t}) - \mathbb{E}_{\tau} [\hat{G}(\theta_t ; \tau_t ; \bar{w}_{K, t})] | \bar{w}_{K, t}] | \mathcal{F}_t] \\
        &= \mathbb{E}[ \mathbb{E}_\tau[\hat{G}(\theta_t ; \tau_t ; \bar{w}_{K, t})] - \mathbb{E}_{\tau} [\hat{G}(\theta_t ; \tau_t ; \bar{w}_{K, t})] | \mathcal{F}_t] = 0
    \end{aligned}
\end{equation*}
Now we want to show the bound on $\hat{G}(\theta_t ; \tau_t)$.
\begin{equation*}
    \begin{aligned}
        \lVert  \hat{G}(\theta_t ; \tau_t ; \bar{w}_{K, t})\lVert  &= \lVert \sum_{j = 0}^H \gamma^j Q_{\bar{w}_{K, t}} (s_j, a_j) \nabla \log \pi_{\theta_{t}} (a_j | s_j)\lVert  \\
        &= \lVert \sum_{j = 0}^H \gamma^j \bar{w}_{K, t} ^\intercal \phi (s_j, a_j) \nabla \log \pi_{\theta_{t}} (a_j | s_j)\lVert  \leq \frac{RG}{1 - \gamma}
    \end{aligned}
\end{equation*}
By Lemma \ref{lemma:sigma}
\begin{equation*}
    \begin{aligned}
        \mathbb{E}[\lVert \xi_{t+1} \lVert ^2 | | \bar{w}_{K, t}] &= \mathbb{E}_\tau[\lVert \xi_{t+1} \lVert ^2] \leq \sigma^2 \\ 
        \mathbb{E}[\lVert \xi_{t+1} \lVert ^2 | \mathcal{F}_t] &= \mathbb{E}[\mathbb{E}[\lVert \xi_{t+1} \lVert ^2 | \bar{w}_{K, t}] | \mathcal{F}_t] \leq \sigma^2 \\
        \mathbb{E}[\lVert \xi_{t+1} \lVert ^4 | \mathcal{F}_t] &\leq 4\sigma^2
    \end{aligned}
\end{equation*}
Finally we have
\begin{equation*}
    \begin{aligned}
        \mathbb{E}[\lVert d_{t+1}\rVert^2 \lVert \xi_{t+1} \rVert^2 | \mathcal{F}_{t} ] &= \mathbb{E}[ \mathbb{E}[\lVert d_{t+1}\rVert^2 \lVert \xi_{t+1} \rVert^2 | \bar{w}_{K, t} ] | \mathcal{F}_{t} ] \\
        &= \mathbb{E}[ \lVert d_{t+1}\rVert^2 \mathbb{E}[ \lVert \xi_{t+1} \rVert^2 | \bar{w}_{K, t} ] | \mathcal{F}_{t} ] \\
        &\leq \mathbb{E}[ \lVert d_{t+1}\rVert^2 | \mathcal{F}_{t} ]  \sigma^2
    \end{aligned}
\end{equation*}
\end{proof}

In the following lemma, we bound the bias in the actor-critic gradient estimator due to truncation of the infinite horizon. The approach is similar to the proof of Lemma \ref{lemma:boundedbias}.

\begin{lemma}
\label{lemma:truncboundAC}
For $D_p = \frac{G R}{(1 - \gamma)}$ and $H \geq \frac{\log \mu}{\log \gamma}$, the truncation bias is deterministically bounded such that
$$\lVert p_{t+1}\lVert  \leq D_p\mu $$
$$\lVert p_{t+1}\lVert ^4 \leq D_p^4\mu^4 .$$
\end{lemma}
\begin{proof}
\begin{equation*}
    \begin{aligned}
        \lVert p_{t+1}\lVert  &= \lVert G_H(\theta_t) - G_{\infty}(\theta_t)\lVert  \\
        &= \lVert \mathbb{E}[\sum_{j = 0}^H \gamma^j Q_{\bar{w}_{K,t}} (s_j, a_j) \nabla \log \pi_{\theta_{t}} (a_j | s_j)] - \mathbb{E}[\sum_{j = 0}^\infty \gamma^j Q_{\bar{w}_{K, t}} (s_j, a_j) \nabla \log \pi_{\theta_{t}} (a_j | s_j)] \lVert \\
        &= \lVert \mathbb{E}[\sum_{j = H}^\infty \gamma^j Q_{\bar{w}_{K, t}} (s_j, a_j) \nabla \log \pi_{\theta_{t}} (a_j | s_j)] \lVert  \\
        &= \lVert \mathbb{E}[\sum_{j = H}^\infty \gamma^j \bar{w}_{K,t}^\intercal \phi(s_j, a_j) \nabla \log \pi_{\theta_{t}} (a_j | s_j)] \lVert  \\
        &\leq R G  \sum_{j = H}^\infty \gamma^j \leq \frac{G R}{(1 - \gamma)}  \gamma^H
    \end{aligned}
\end{equation*}
\end{proof}

\subsection{Proof of Lemma \ref{lemma:FAbiasboundAC}}
\label{sec:FAbiasboundAC}

Although Theorem \ref{theorem:bhandarialt} establishes the convergence of TD(0) in terms of $Q_{\bar{w}_K}$ under constant step sizes, we actually require a stronger convergence result showing the direct convergence of $\bar{w}_K$ to $w^*$ in order to bound the critic approximation bias. In the following proofs, we use core results proved in Appendix \ref{appendix:td(0)} to prove an alternate fourth-moment bound under diminishing step sizes. These convergence results are formalized in Lemma \ref{lemma:fourthmomentAC}.

\begin{lemma}
\label{lemma:fourthmomentAC}
    Suppose $\bar{w}_K$ is generated by $K$ steps of the Projected TD(0) algorithm with $w^* \in \Theta$ and diminishing step-size $\alpha_t~=~\frac{1}{(t + 1) \varsigma}$. Then
$$\mathbb{E}[\lVert w^* - \bar{w}_{K}\lVert ^4] \leq \frac{\log^2 K}{K} \cdot \left(\frac{192 F^2 R^2}{\varsigma^2 \log^2(r^{-1})} + O(\frac{1}{\log K}) + O(\frac{1}{\log^2 K})\right) $$
\end{lemma}

\textit{Proof:} See Appendix \ref{sec:fourthmomentAC}

Lemma \ref{lemma:finalfourthmomentACmu} follows from Lemma \ref{lemma:fourthmomentAC}. 

\begin{lemma}
\label{lemma:finalfourthmomentACmu}
    For $K = O(\frac{\log^2(\mu^{-4})}{\mu^4})$ as $\mu \to 0$ iterations of Algorithm \ref{alg:acgradientestimator} and
$$D = \left(\frac{192 F^2 R^2}{\varsigma^2 \log^2(r^{-1})} + O(\frac{1}{\log \mu^{-4}}) \right)^{-1/4}$$
we have
$$\mathbb{E}[\lVert w^* - \bar{w}_{K}\lVert ^4] \leq D^4 \mu^4 $$
\end{lemma}

\textit{Proof.} See Appendix \ref{sec:whatdoievencallthese}.

Now we can proceed with the proof of Lemma \ref{lemma:FAbiasboundAC}.
\begin{proof}
We can characterize the bias by the quality of the critic approximation achieved by the TD(0) algorithm. 
 We can "roll up" the temporal summation as follows
\begin{equation*}
    \begin{aligned}
        q_{t+1} &= G_\infty - \nabla J(\theta_t) \\
        &= \mathbb{E}_{\pi, \rho_0} [\sum_{k = 0}^\infty \gamma^k Q_{\bar{w}_{K,t}}(s_k, a_k) \nabla \log \pi_{\theta} (a_k | s_k)] - \nabla J(\theta) \\
    &= \mathbb{E}_{\pi, \rho_0} [\sum_{k = 0}^\infty \gamma^k Q_{\bar{w}_{K, t}}(s_k, a_k) \nabla \log \pi_{\theta} (a_k | s_k)] - \mathbb{E}_{\pi, \rho_0} [\sum_{k = 0}^\infty \gamma^k \nabla \log \pi(a_k|s_k) Q_\gamma^\pi (s_k, a_k)] \\
    &= \mathbb{E}_{\pi, \rho_0} [\sum_{k = 0}^\infty \gamma^k  \nabla \log \pi_{\theta} (a_k | s_k) (Q_{\bar{w}_{K, t}}(s_k, a_k) - Q_\gamma^\pi (s_k, a_k))]  \\
   &= \sum_{s \in \mathcal{S}} \sum_{a \in \mathcal{A}} \sum_{k = 0}^\infty \mathbb{P} (s_k = s | s_0) \pi(a | s) \gamma^k  \nabla \log \pi_{\theta} (a | s) (Q_{\bar{w}_{K, t}}(s, a) - Q_\gamma^\pi (s, a)) \\
   &= \sum_{s \in \mathcal{S}} \sum_{a \in \mathcal{A}}  d_\gamma^\pi(s) \pi(a | s)  \nabla \log \pi_{\theta} (a | s) (Q_{\bar{w}_{K, t}}(s, a) - Q_\gamma^\pi (s, a)), \\
    \end{aligned}
\end{equation*}
where $d^\pi_\gamma(s) = \sum_{k = 0}^\infty \gamma^k \mathbb{P} (s_k = s | s_0)$ denotes the discounted state visitation measure. Since we have
\begin{equation*}
\sum_{s \in \mathcal{S}} \sum_{a \in \mathcal{A}}  d_\gamma^\pi(s) \pi(a | s) = \sum_{s \in \mathcal{S}}  d_\gamma^\pi(s) = \frac{1}{1 - \gamma},
\end{equation*}
then by Jensen's inequality, we can obtain the following bound 
\begin{equation*}
    \begin{aligned}
        \lVert q_{t+1}\lVert ^4 &\leq (1 - \gamma) \sum_{s \in \mathcal{S}} \sum_{a \in \mathcal{A}}  d_\gamma^\pi(s) \pi(a | s)  \lVert \nabla \log \pi_{\theta} (a | s) (Q_{\bar{w}_{K, t}}(s, a) - Q_\gamma^\pi (s, a))\lVert ^4 \\
        &\leq (1 - \gamma)  \sum_{s \in \mathcal{S}} \sum_{a \in \mathcal{A}}  d_\gamma^\pi(s) \pi(a | s)  G^4 \lVert Q_{\bar{w}_{K, t}}(s, a) - Q_\gamma^\pi (s, a)\lVert ^4 \\
        &= (1 - \gamma)  \sum_{s \in \mathcal{S}} \sum_{a \in \mathcal{A}}  d_\gamma^\pi(s) \pi(a | s)  G^4 \lVert \bar{w}_{K, t}^{\intercal} \phi(s, a) - {w^*}^{\intercal} \phi(s, a)\lVert ^4 \\
        &\leq (1 - \gamma)  \sum_{s \in \mathcal{S}} \sum_{a \in \mathcal{A}}  d_\gamma^\pi(s) \pi(a | s)  G^4 \lVert \phi(s, a)\lVert ^4\cdot\lVert \bar{w}_{K, t}  - {w^*}\lVert ^4 \\
        &\leq (1 - \gamma)  \sum_{s \in \mathcal{S}} \sum_{a \in \mathcal{A}}  d_\gamma^\pi(s) \pi(a | s)  G^4 \cdot\lVert \bar{w}_{K, t}  - {w^*}\lVert ^4 \\
        &\leq G^4 \cdot\lVert \bar{w}_{K, t}  - {w^*}\lVert ^4. \\
    \end{aligned}
\end{equation*}
Now we take expectation on both sides with respect to the TD(0) algorithm that occurs between timestep $t$ and $t+1$ to obtain
\begin{equation*}
    \mathbb{E}[\lVert q_{t+1}\lVert ^4] \leq   G^4 \cdot \mathbb{E}[\lVert \bar{w}_{K, t}  - {w^*}\lVert ^4].
\end{equation*}
Finally, we can apply Lemma \ref{lemma:finalfourthmomentACmu} to achieve the final result.
\end{proof}

\subsection{Proof of Key Lemmas}
\subsubsection{Proof of Lemma \ref{lemma:fourthmomentAC}}
\label{sec:fourthmomentAC}
We first require the following auxiliary lemma.

\begin{lemma}
    \label{lemma:helperTD}
    Recall the definition of $A_\theta = \mathbb{E} _ {(s, a) \sim \eta_\theta, (s', a') \sim P(s, a, \cdot)}[\phi(s, a) (\phi(s, a) - \gamma \phi(s', a'))^{\intercal}]$. Then 
$$(w^* - w_t)^\intercal\bar{g}(w_t) \geq \frac{1}{2} \lambda_{min}(A_\theta + A_\theta^T) \lVert w^* - w_t\lVert ^2 \geq \frac{1}{2} \varsigma \lVert w_t - w^*\lVert ^2$$
where $\varsigma$ is defined in Assumption \ref{ass:boundedeigenvalue}.
\end{lemma}

\begin{proof}
\begin{equation*}
    \begin{aligned}
        (w^* - w_t)^\intercal\bar{g}(w_t) &= (w^* - w_t)^\intercal(\bar{g}(w_t) - \bar{g}(w^*)) \\
         &= (w^* - w_t)^\intercal\mathbb{E}[(\gamma \phi(s', a')^\intercal w_t - \phi (s, a)^\intercal w_t - \gamma \phi(s', a')^\intercal w^* - \phi (s, a)^\intercal w^*) \phi(s, a)]\\
         &= (w^* - w_t)^\intercal\mathbb{E}[(\gamma \phi(s', a') - \phi (s, a))^\intercal  (w_t - w^*) \phi(s, a)] \\
         &= (w^* - w_t)^\intercal A_\theta (w^* - w_t) \\
         &= \frac{1}{2} (w^* - w_t)^\intercal (A_\theta + A_\theta^T) (w^* - w_t)\\
         &\geq \frac{1}{2} \lambda_{min}(A_\theta + A_\theta^T) \lVert w^* - w_t\lVert ^2
    \end{aligned}
\end{equation*}
\end{proof}

Now we can proceed with the proof of Lemma \ref{lemma:fourthmomentAC}.

\begin{proof}
\begin{equation*}
    \begin{aligned}
        \lVert  w^* - w_{t+1}\lVert ^2 &= \lVert  w^* - Proj_{\Theta}(w_{t} + \alpha_t g_t(w_t))\lVert ^2 \\
        &= \lVert  Proj_{\Theta}(w^*) - Proj_{\Theta}(w_{t} + \alpha_t g_t(w_t))\lVert ^2 \\
        &\leq \lVert w^* - w_{t} - \alpha_t g_t(w_t)\lVert ^2 \\
        &= \lVert w^* - w_{t}\lVert ^2 - 2 \alpha_t g_t(w_t)^{\intercal}(w^* - w_t)+ \alpha_t^2 \lVert g_t(w_t)\lVert ^2 \\
        &= \lVert w^* - w_{t}\lVert ^2 - 2 \alpha_t \bar{g}(w_t)^{\intercal}(w^* - w_t) + 2 \alpha_t \zeta_t(w_t) + \alpha_t^2 F^2\\
    \end{aligned}
\end{equation*}
So we have the following fourth moment bound (since both sides of the inequality are positive):
\begin{equation*}
    \begin{aligned}
        \lVert  w^* - w_{t+1}\lVert ^4 &\leq [\lVert w^* - w_{t}\lVert ^2 - 2 \alpha_t \bar{g}(w_t)^{\intercal}(w^* - w_t) + 2 \alpha_t \zeta_t(w_t) + \alpha_t^2 F^2 \big]^2\\
        &= \lVert w^* - w_{t}\lVert ^4 + 4 \alpha_t^2 (\bar{g}(w_t)^{\intercal}(w^* - w_t) )^2 + 4 \alpha_t^2 (\zeta_t(w_t))^2 + \alpha_t^4 F^4 \\
        &- 4 \alpha_t\lVert w^* - w_{t}\lVert ^2 (\bar{g}(w_t)^{\intercal}(w^* - w_t) ) + 4 \alpha_t \lVert w^* - w_{t}\lVert ^2 \zeta_t(w_t) + 2 \alpha_t^2 F^2 \lVert w^* - w_{t}\lVert ^2 \\
        &- 8 \alpha_t^2 (\bar{g}(w_t)^{\intercal}(w^* - w_t) ) \zeta_t(w_t) - 4 \alpha_t^3 F^2  (\bar{g}(w_t)^{\intercal}(w^* - w_t) ) + 4 \alpha_t^3 F^2 \zeta_t(w_t) \\
    \end{aligned}
\end{equation*}
By Lemma \ref{lemma:helperTD}, 
\begin{equation*}
    \begin{aligned}
        \lVert  w^* - w_{t+1}\lVert ^4 \leq& \lVert w^* - w_{t}\lVert ^4 + 4 \alpha_t^2 (\bar{g}(w_t)^{\intercal}(w^* - w_t) )^2 + 4 \alpha_t^2 (\zeta_t(w_t))^2 + \alpha_t^4 F^4 \\
        &- 2 \alpha_t\lVert w^* - w_{t}\lVert ^2 \varsigma \lVert w^* - w_t\lVert ^2 + 4 \alpha_t \lVert w^* - w_{t}\lVert ^2 \zeta_t(w_t) + 2 \alpha_t^2 F^2 \lVert w^* - w_{t}\lVert ^2 \\
        &+ 4 \alpha_t^2 \varsigma \lVert w^* - w_t\lVert ^2 |\zeta_t(w_t)| - 2 \alpha_t^3 F^2  \varsigma \lVert w^* - w_t\lVert ^2 + 4 \alpha_t^3 F^2 \zeta_t(w_t) \\
        \leq& (1 - 2 \alpha_t \varsigma) \lVert w^* - w_{t}\lVert ^4 + 4 \alpha_t^2 (\bar{g}(w_t)^{\intercal}(w^* - w_t) )^2 + 4 \alpha_t^2 (\zeta_t(w_t))^2 + \alpha_t^4 F^4 \\
        & + (4 \alpha_t +  4 \alpha_t^2 \varsigma)\lVert w^* - w_{t}\lVert ^2 |\zeta_t(w_t)| + (2 \alpha_t^2 F^2 - 2 \alpha_t^3 F^2  \varsigma ) \lVert w^* - w_{t}\lVert ^2 + 4 \alpha_t^3 F^2 \zeta_t(w_t) .\\
    \end{aligned}
\end{equation*}
We can utilize the bounds $\lVert w\lVert  \leq R$, $|\zeta_t(w) | \leq 2F^2$, $\lVert g_t(w)\lVert  \leq F$ from Lemma \ref{lemma:6and10} to arrive at
\begin{equation*}
    \begin{aligned}
        \lVert  w^* - w_{t+1}\lVert ^4 \leq& (1 - 2 \alpha_t \varsigma) \lVert w^* - w_{t}\lVert ^4 + 16 \alpha_t^2 F^2 R^2 + 16 \alpha_t^2 F^4 + \alpha_t^4 F^4 \\
        & + (4 \alpha_t +  4 \alpha_t^2 \varsigma)\lVert w^* - w_{t}\lVert ^2 |\zeta_t(w_t)| + (2 \alpha_t^2 F^2 - 2 \alpha_t^3 F^2  \varsigma ) 4R^2 + 8 \alpha_t^3 F^4 \\
        \leq& (1 - 2 \alpha_t \varsigma) \lVert w^* - w_{t}\lVert ^4 + (8 \alpha_t +  8 \alpha_t^2 \varsigma)R^2|\zeta_t(w_t)| + 24 \alpha_t^2 F^2 R^2 + 16 \alpha_t^2 F^4 \\
        &+ \alpha_t^4 F^4  - 8 \alpha_t^3 F^2  \varsigma  R^2 + 8 \alpha_t^3 F^4 .\\
    \end{aligned}
\end{equation*}
After some rearrangement of terms, we have
\begin{equation*}
    \begin{aligned}
          \alpha_t \varsigma \lVert w^* - w_{t}\lVert ^4 \leq& (1 -  \alpha_t \varsigma) \lVert w^* - w_{t}\lVert ^4 - \lVert  w^* - w_{t+1}\lVert ^4 + ( 8\alpha_t +  8 \alpha_t^2 \varsigma)R^2|\zeta_t(w_t)| \\
        &+ 24 \alpha_t^2 F^2 R^2 + 16 \alpha_t^2 F^4 + \alpha_t^4 F^4  - 8 \alpha_t^3 F^2  \varsigma  R^2 + 8 \alpha_t^3 F^4 \\
        \mathbb{E}[\lVert w^* - w_{t}\lVert ^4] \leq& (\frac{1}{\alpha_t \varsigma } -  1) \mathbb{E}[\lVert w^* - w_{t}\lVert ^4] - \frac{1}{\alpha_t \varsigma }\mathbb{E}[\lVert  w^* - w_{t+1}\lVert ^4] + (\frac{8}{\varsigma} +  8 \alpha_t)R^2 \mathbb{E}[|\zeta_t(w_t)|] \\
        &+ \frac{24 \alpha_t F^2 R^2}{\varsigma} + \frac{16 \alpha_t F^4}{\varsigma} + \frac{\alpha_t^3 F^4}{\varsigma}  - 8 \alpha_t^2 F^2  R^2 + \frac{8 \alpha_t^2 F^4 }{\varsigma}.\\
    \end{aligned}
\end{equation*}
Let $\alpha_t = \frac{1}{(t + 1) \varsigma}$, then we have 
\begin{equation*}
    \begin{aligned}
        \mathbb{E}[\lVert w^* - w_{t}\lVert ^4] \leq& t \mathbb{E}[\lVert w^* - w_{t}\lVert ^4] - (t + 1)\mathbb{E}[\lVert  w^* - w_{t+1}\lVert ^4] + (\frac{8}{\varsigma} +  \frac{8}{(t + 1) \varsigma})R^2 \mathbb{E}[|\zeta_t(w_t)|] \\
        &+ \frac{24 F^2 R^2}{\varsigma^2 (t + 1) } + \frac{16 F^4}{\varsigma^2 (t + 1)} + \frac{ F^4}{\varsigma^4 (t + 1)^3}  - \frac{8 F^2  R^2}{(t + 1)^2 \varsigma^2 } + \frac{8 F^4 }{\varsigma^3 (t + 1)^2}.\\
    \end{aligned}
\end{equation*}
Then we sum on either side from $0$ to $K - 1$ and divide by $K$, using the facts that $\sum_{t = 1}^K \frac{1}{t^2} \leq~\frac{\pi^2}{6}$ and $\sum_{t = 1}^K \frac{1}{t} = \log(K) + O(1)$ to conclude
\begin{equation*}
    \begin{aligned}
        \frac{1}{K} \sum_{t = 0}^{K -1} \mathbb{E}[\lVert w^* - w_{t}\lVert ^4] \leq& \frac{1}{K} \sum_{t = 0}^{K -1} [t \mathbb{E}[\lVert w^* - w_{t}\lVert ^4] - (t + 1)\mathbb{E}[\lVert  w^* - w_{t+1}\lVert ^4] ] + \frac{1}{K} \sum_{t = 0}^{K -1} (\frac{8}{\varsigma} +  \frac{8}{(t + 1) \varsigma})R^2 \mathbb{E}[|\zeta_t(w_t)|] \\
        &+ \frac{1}{K} \sum_{t = 0}^{K -1} [\frac{24 F^2 R^2}{\varsigma^2 (t + 1) } + \frac{16 F^4}{\varsigma^2 (t + 1)} + \frac{ F^4}{\varsigma^4 (t + 1)^2}  + \frac{8 F^4 }{\varsigma^3 (t + 1)^2} ]\\
        \leq& \frac{\lVert w_1 - w^*\lVert ^4}{K}  + \frac{1}{K} \sum_{t = 0}^{K -1} (\frac{8}{\varsigma} +  \frac{8}{(t + 1) \varsigma})R^2 \mathbb{E}[|\zeta_t(w_t)|] + \frac{1}{K} \sum_{t = 0}^{K -1} [\frac{24 F^2 R^2}{\varsigma^2 (t + 1) } + \frac{16 F^4}{\varsigma^2 (t + 1)} ]\\
        &+ \frac{4 F^4 \pi^2 }{3 \varsigma^3 K} + \frac{F^4 \pi^2 }{6 \varsigma^4 K }\\
        \leq& \frac{\lVert w_1 - w^*\lVert ^4}{K}  + \frac{1}{K} \sum_{t = 0}^{K -1} \frac{16 R^2}{\varsigma} \mathbb{E}[|\zeta_t(w_t)|] + (\frac{24 F^2 R^2 + 16 F^4}{\varsigma^2 }) \frac{\log K + O(1)}{K} \\
        &+ \frac{4 F^4 \pi^2 }{3 \varsigma^3 K } + \frac{F^4 \pi^2 }{6 \varsigma^4 K}.\\
    \end{aligned}
\end{equation*}
To bound the summation on the right hand side, we can apply the results of Lemma \ref{lemma:modified11} for $t \leq K-1$, with $\tau_0~=~\tau^{\text{\tiny mix}} (\alpha_{T-1}) \leq \frac{\log(m \varsigma T)}{\log(r^{-1})}$. Then we have
\begin{equation*}
    \begin{aligned}
        \sum_{t = 0}^{K - 1}\mathbb{E}[|\zeta_t (w_t)|] &\leq F^2(8 + 6 \tau_0)  \sum_{t = 0}^{K - 1}\alpha_{t^*} + 10 F^2 m \sum_{t = 0}^{K - 1} r^t \\
        &= F^2(8 + 6 \tau_0) \sum_{t = 0}^{\tau_0} \alpha_0 + F^2(8 + 6 \tau_0) \sum_{t = \tau_0 + 1}^{K-1} \alpha_t +  10 F^2 m \sum_{t = 0}^{K - 1} r^t \\
        &\leq F^2(8 + 6 \tau_0) \sum_{t = 0}^{\tau_0} \alpha_0 + F^2(8 + 6 \tau_0) \sum_{t = \tau_0 + 1}^{K-1} \alpha_t +  \frac{10 F^2 m}{1 - r} \\
        &= F^2(8 + 6 \tau_0) \tau_0 \alpha_0 + F^2(8 + 6 \tau_0) \sum_{t = \tau_0 + 1}^{K-1} \alpha_t +  \frac{10 F^2 m}{1 - r} \\
        &\leq \frac{8F^2 }{\varsigma} \frac{\log(m \varsigma K)}{\log(r^{-1})} + \frac{6 F^2}{\varsigma }\frac{\log^2(m \varsigma K)}{\log^2(r^{-1})} + \frac{F^2}{\varsigma} (8 + 6 \frac{\log(m \varsigma K)}{\log(r^{-1})})(\log K + O(1)) + \frac{10 F^2 m}{1 - r}\\
        &\leq (\frac{6 F^2}{\varsigma \log^2 (r^{-1})} + \frac{6 F^2}{\varsigma \log (r^{-1})}) \log^2 K + O(\log(K)) + O(1) \\
        &\leq \frac{12 F^2}{\varsigma \log^2 (r^{-1})} \log^2 K + O(\log(K)) + O(1) .\\
    \end{aligned}
\end{equation*}
So we have for the original expression 
\begin{equation*}
    \begin{aligned}
         \frac{1}{K} \sum_{t = 0}^{K -1} \mathbb{E}[\lVert w^* - w_{t}\lVert ^4] \leq&  (\frac{192 F^2 R^2}{\varsigma^2 \log^2(r^{-1})} + O(\frac{1}{\log K}) + O(\frac{1}{\log^2 K})) \cdot \frac{\log^2 K}{K},
         \end{aligned}
\end{equation*}
and by Jensen's inequality, we obtain
\begin{equation*}
    \mathbb{E}[\lVert w^* - \bar{w}_{K}\lVert ^4] \leq \frac{1}{K} \sum_{t = 0}^{K -1} \mathbb{E}[\lVert w^* - w_{t}\lVert ^4] \leq 
    (\frac{192 F^2 R^2}{\varsigma^2 \log^2(r^{-1})} + O(\frac{1}{\log K}) + O(\frac{1}{\log^2 K})) \cdot \frac{\log^2 K}{K} .\\
\end{equation*}
\end{proof}

\subsubsection{Proof of Lemma \ref{lemma:finalfourthmomentACmu}}
\label{sec:whatdoievencallthese}

\textit{Proof.} We establish a bound on $\mathbb{E}[\lVert w^* - \bar{w}_{K}\lVert ^4]$ via Lemma \ref{lemma:fourthmomentAC}. Let $\epsilon = \mu^2$. Then we want to find $K$ large enough such that 
\begin{equation*}
\log(K) \leq \epsilon  \sqrt{K}    
\end{equation*}
We look for the asymptotic solution to 
\begin{equation*}
\log(K) = \epsilon \sqrt{K}    
\end{equation*}
in the form $K = \frac{k_1}{\epsilon^2}$. Plugging this in, we get
\begin{equation*}
\log(k_1) + \log(\frac{1}{\epsilon^2}) = \sqrt{k_1}    
\end{equation*}
By the method of dominant balance, we have
\begin{equation*}
    \begin{aligned}
        \sqrt{k_1} &\approx \log (\frac{1}{\epsilon^2}) \\
        k_1 &\approx \log^2(\frac{1}{\epsilon^2})
    \end{aligned}
\end{equation*}
So we have $K = O(\frac{ \log^2(\frac{1}{\epsilon^2})}{\epsilon^2})$.

\section{Policy Gradient Theorem}
\label{PGT}

The original policy gradient theorem derived in \cite{1992SuttonOG} addresses the gradient of the value function, which is a slightly different objective that we denote as $\tilde{J}(\theta)$ as follows
$$\tilde{J}(\theta) = V^{\pi_\theta}(s_0) = \mathbb{E}_{\pi_\theta}[\sum_{k = 0}^\infty \gamma^k \mathcal{R}(s_k, a_k) | s_0] $$
Then from \cite{1992SuttonOG} we have the original policy gradient theorem:
$$\nabla \tilde{J} (\theta) = \sum_{s \in \mathcal{S}} d^\pi (s) \sum_{a \in \mathcal{A}}\nabla \pi_\theta (s, a) Q^{\pi_\theta} (s, a) $$
Where
$$d^{\pi} (s) = \sum_{k = 0}^\infty \gamma^k Pr(s_k = s | s_0, \pi)$$
We consider instead the expectation of the value function over an initial state distribution:
$$J(\theta) = \mathbb{E}_{\rho_0}[V^{\pi_\theta}(s_0)]$$
Then
$$\nabla J(\theta) = \sum_{s \in \mathcal{S}} d^{\pi_\theta, \rho_0} (s) \sum_{a \in \mathcal{A}} \nabla \pi_\theta (a | s) Q ^\pi (s, a)$$
Where
$$d^{\pi_\theta, \rho_0} (s) = \sum_{k = 0}^\infty \gamma^k \mathbb{E}_{\rho_0}[Pr(s_k = s | s_0, \pi_\theta)]$$
When we implement the ``log-likelihood trick", we have
$$\nabla J(\theta) = \sum_s d^{\pi_\theta, \rho_0} \sum_a  \pi_\theta (a | s) \nabla \log \pi_\theta (a | s) Q ^{\pi_\theta} (s, a)$$
We can ``unroll" this result as in \cite{wu22sensitivity} to acquire the temporal formulation:
$$= \mathbb{E}_{\pi, \rho_0} [\sum_{k = 0}^\infty \gamma^k \nabla \log \pi(a_k|s_k) Q^\pi (s_k, a_k)]$$
To derive the GPOMDP estimator from this result, we use the definition of $Q^\pi$
$$Q^\pi(s, a) = \mathbb{E}_\pi[\sum_{t = 0}^\infty \gamma^t \mathcal{R}(s_t, a_t)| s_0 = s, a_0 = a] = \mathbb{E}_\pi[\sum_{t = 0}^\infty \gamma^t \mathcal{R}(s_{t+k}, a_{t+k})| s_k = s, a_k = a]$$
\begin{equation*} \begin{aligned}
\nabla J(\theta) &= \mathbb{E}_{\pi, \rho_0} [\sum_{k = 0}^\infty \gamma^k \nabla \log \pi(a_k|s_k) \mathbb{E}_\pi[\sum_{t = 0}^\infty \gamma^t R(s'_{t+k}, a'_{t+k}) | s'_k = s_k, a'_k = a_k]] \\
&= \mathbb{E}_{\pi, \rho_0} [\sum_{k = 0}^\infty \mathbb{E}_\pi[\sum_{t = 0}^\infty \gamma^{k + t} \nabla \log \pi(a_k|s_k) \mathcal{R}(s'_{t+k}, a'_{t+k}) | s'_k = s_k, a'_k = a_k]] \\
&= \mathbb{E}_{\pi, \rho_0} [\sum_{k = 0}^\infty \sum_{t = 0}^\infty \gamma^{k + t} \nabla \log \pi(a_k|s_k) \mathcal{R}(s_{t+k}, a_{t+k}) ] \\
&= \mathbb{E}_{\pi, \rho_0} [\sum_{k = 0}^\infty \sum_{t = k}^\infty \gamma^{t} \nabla \log \pi(a_k|s_k) \mathcal{R}(s_{t}, a_{t}) ] 
\end{aligned} \end{equation*} 
And so we end with an unbiased estimator of the policy gradient 
$$\nabla J(\theta) = \mathbb{E}_{\tau \sim p (\cdot | \theta)} [ \nabla \log p_\theta(\tau) \mathcal{R}(\tau)].$$


\end{document}